\documentclass[11pt]{article}

\usepackage[margin=1in]{geometry}
\usepackage{algorithm}
\usepackage{algorithmic}
\usepackage{amsmath, amssymb, amsthm}
\usepackage{hyperref}
\usepackage[nameinlink,capitalize,noabbrev]{cleveref}
\usepackage{graphicx}
\usepackage{natbib}
\usepackage{url}            
\usepackage{booktabs}       
\usepackage{amsfonts}       
\usepackage{nicefrac}       
\usepackage{microtype}      
\usepackage{xcolor}         
\usepackage{mathtools}
\usepackage{algorithm}
\usepackage{dsfont}
\usepackage{bbm}
\usepackage{longtable}
\usepackage{multirow}

\theoremstyle{plain}
\newtheorem{theorem}{Theorem}[section]

\newtheorem{lemma}[theorem]{Lemma}
\newtheorem{corollary}[theorem]{Corollary}
\theoremstyle{definition}
\newtheorem{definition}[theorem]{Definition}
\newtheorem{assumption}[theorem]{Assumption}
\theoremstyle{remark}
\newtheorem{remark}[theorem]{Remark}

\title{Exact Finite-Sample Variance Decomposition of Subagging: A Spectral Filtering Perspective}

\author{
Ye Su\\
Shenzhen Institutes of Advanced Technology\\
Chinese Academy of Sciences\\
Shenzhen 518055, China\\
\texttt{ye.su@siat.ac.cn}
\and
Mingrui Ye\\
School of Mathematics\\
Renmin University of China\\
Beijing 100872, China\\
\texttt{mingruiye2005@ruc.edu.cn}
\and
Yining Wang\\
School of Mathematics\\
Renmin University of China\\
Beijing 100872, China\\
\texttt{wangyining@ruc.edu.cn}
\and
Jipeng Guo\\
College of Information Science and Technology\\
Beijing University of Chemical Technology\\
Beijing 100029, China\\
\texttt{guojipeng@buct.edu.cn}
\and
Yong Liu\thanks{Corresponding author}\\
Gaoling School of Artificial Intelligence\\
Renmin University of China\\
Beijing 100872, China\\
\texttt{liuyonggsai@ruc.edu.cn}
}

\date{}

\begin{document}

\maketitle

\begin{abstract}
Standard resampling ratios (e.g., $\alpha \approx 0.632$) are widely used as default baselines in ensemble learning for three decades. However, how these ratios interact with a base learner's intrinsic functional complexity in finite samples lacks a exact mathematical characterization. We leverage the Hoeffding-ANOVA decomposition to derive the first exact, finite-sample variance decomposition for subagging, applicable to any symmetric base learner without requiring asymptotic limits or smoothness assumptions. We establish that subagging operates as a deterministic low-pass spectral filter: it preserves low-order structural signals while attenuating $c$-th order interaction variance by a geometric factor approaching $\alpha^c$. This decoupling reveals why default baselines often under-regularize high-capacity interpolators, which instead require smaller $\alpha$ to exponentially suppress spurious high-order noise. To operationalize these insights, we propose a complexity-guided adaptive subsampling algorithm, empirically demonstrating that dynamically calibrating $\alpha$ to the learner's complexity spectrum consistently improves generalization over static baselines.
\end{abstract}

\section{Introduction}

Ensemble methods, particularly Bagging \citep{breiman1996bagging} and Subagging (subsample aggregating), remain foundational to machine learning by reducing the prediction variance of unstable base learners without inflating bias \citep{breiman2001using, sagi2018ensemble, zhou2021ensemble}. Historically, algorithmic implementations default to standard resampling ratios. Bootstrap sampling yields approximately $63.2\%$ unique training instances per base learner, driven by asymptotic exclusion probabilities \citep{breiman1996out, efron1997improvements}. For subagging (sampling without replacement), fixed ratios like $\alpha = 0.5$ or $0.8$ have served as heuristic baselines for nearly three decades \citep{buhlmann2002analyzing}. 

Early theoretical formalizations \citep{buhlmann2002analyzing} provided elegant asymptotic justifications for ensemble variance reduction, modeling subagging as a macroscopic smoothing operation. While intuitive, these approximations treat the base learner's variance as an aggregate quantity. However, modern high-capacity over-parameterized models (e.g., deep unpruned decision trees) frequently interpolate data via complex, high-order multi-sample interactions \citep{wyner2017explaining, belkin2019reconciling, zhang2016understanding}. For such models, variance is not uniform but heavily skewed toward high-frequency components.

This paradigm shift exposes a critical gap: while practitioners acknowledge the subsampling ratio $\alpha$ as a tunable hyperparameter, its optimization typically relies on blind hyperparameter search \citep{martinez2010out}. The literature lacks a deterministic mechanism linking the optimal ratio to the base learner's functional complexity. Mathematically, a subagged symmetric base learner is an incomplete, infinite-order U-statistic \citep{hoeffding1992class, lee2019u}. Traditional variance expansions evaluate covariance based on macroscopic sample overlaps. Applying the Hoeffding-ANOVA decomposition here yields intractable hypergeometric summations and severe combinatorial explosions across interaction orders. Consequently, foundational works \citep{mentch2016quantifying, wager2018estimation} and recent advances in variance estimation \citep{xu2024variance} resort to asymptotic limits. These limits collapse high-order structural dependencies into first-order H\'{a}jek projections, completely bypassing finite-sample combinatorics.

While effective for statistical inference, this asymptotic simplification is inadequate for analyzing modern interpolators. By projecting variance onto first-order main effects, the asymptotic limit eliminates the high-order structural dependencies where interpolation noise is concentrated \citep{wyner2017explaining}. Therefore, existing frameworks bypass the finite-sample filtering mechanisms required to theoretically calibrate resampling ratio $\alpha$. To shift from unguided trial-and-error to a theoretically principled adaptation, we must answer two mechanistic questions:
\begin{itemize}
    \item \textbf{\textit{Exactly which functional components of a non-parametric model's variance are reduced by resampling, and which are preserved?}}
    \item \textbf{\textit{By what exact algebraic proportion is each interaction component attenuated in finite samples?}}
\end{itemize} 

In this paper, we directly answer these questions. Methodologically, by bypassing the hypergeometric summations via probabilistic indicator decoupling, we establish that subagging does not reduce variance uniformly. Instead, it operates as a deterministic low-pass spectral filter. It preserves low-order structural signals while attenuating the $c$-th order interaction variance by an exact combinatorial factor bounded by $\alpha^c$. This decoupling reveals why three-decade-old default resampling baselines (e.g., $\alpha \approx 0.632$) often under-regularize high-capacity interpolators, which instead require a smaller $\alpha$ to suppress high-order noise. Our main contributions are summarized as follows:

\begin{itemize}
    \item \textbf{Exact Finite-Sample Variance Identity:} We derive the first exact variance decomposition for subagging (\textcolor{red}{Theorem~\ref{thm:exact_variance}}). This identity is applicable to any symmetric base learner for any finite sample size $n$ and subsample size $k \le n$, without requiring smoothness assumptions or asymptotic limits.
    
    \item \textbf{The Spectral Filtering Law:} We analytically characterize the exact attenuation factor $\gamma_c(n,k)$ for each interaction order $c$ (\textcolor{red}{Theorem~\ref{thm:spectral_filtering}}). We prove that subagging attenuates the $c$-th order variance component by a combinatorial factor bounded by, and asymptotically converging to, $\alpha^c$. 
    
    \item \textbf{Structural Shift of Optimal Regularization:} By integrating the exact variance identity with non-parametric bias rates, we establish a monotone comparative statics result (\textcolor{red}{Theorem~\ref{thm:optimal_alpha}}). We show that the optimal subsampling ratio $\alpha^*$ is monotonically non-increasing with the learner's interaction complexity. This formalizes the transition in the bias-variance trade-off: low-capacity smoothers require larger $\alpha$ for bias control, whereas models dominated by high-order interactions necessitate smaller $\alpha$ to suppress noise. Consequently, this result bridges finite-sample resampling heuristics with the asymptotic consistency requirement $\alpha \to 0$.
    
    \item \textbf{Complexity-Guided Adaptive Subsampling and Empirical Validation:} To operationalize this theory without prohibitive full-spectrum estimations, we propose Complexity-Guided Adaptive Subsampling (CGAS). We empirically validate our structural shift theorem, demonstrating that dynamically calibrating $\alpha$ to the learner's structural capacity systematically improves finite-sample generalization over static baselines.
\end{itemize}

\section{Related Work}
\label{sec:related_work}

\subsection{Classical Bias-Variance Analysis and Double Descent}
The heuristic that bagging reduces variance without inflating bias originated with \citep{breiman1996bagging}. Early theoretical formalizations \citep{buhlmann2002analyzing} provided asymptotic justifications by modeling bagging as a smoothing operation over non-differentiable functional components. These continuous approximations, however, inherently rely on the asymptotic regime. The recent success of highly over-parameterized models has renewed interest in the exact mechanics of ensembling, particularly concerning the interpolation regime and double descent phenomena \citep{wyner2017explaining,belkin2019reconciling,hastie2022surprises}. To define exactly \textbf{\textit{which functional components of a model's variance are reduced by resampling, and by what exact proportion}}, the analysis must shift from asymptotic smoothing to discrete combinatorial decoupling.

\subsection{Infinite-Order U-Statistics and Finite-Sample Bottlenecks}
Mathematically, a subagged estimator constructed from a symmetric base learner is equivalent to an incomplete, infinite-order U-statistic \citep{hoeffding1992class,lee2019u}. The Hoeffding-ANOVA decomposition has been pivotal in analyzing Random Forests \citep{mentch2016quantifying, wager2018estimation}, primarily to establish asymptotic distributional limits (e.g., Gaussianity). Recently, progress has been made in understanding the variance of these estimators. \cite{zhou2021v} explored V-statistic representations to analyze resampling variations, while \cite{xu2024variance} utilized the Hoeffding decomposition to establish ratio-consistent variance estimators for large subsample sizes. Furthermore, \cite{o2026random} developed a finite-sample variance identity specific to random forests. 

Despite these advances, deriving a generic, finite-sample variance decomposition for \textit{any} symmetric base learner has remained a notoriously difficult combinatorial problem. The mathematical bottleneck stems from the macroscopic evaluation of covariance: the overlap between two random subsamples follows a hypergeometric distribution. Integrating the infinite-order Hoeffding expansions over this hypergeometric distribution yields intractable, nested combinatorial summations. Therefore, existing works either resort to asymptotic limits, which collapse high-order dependencies into first-order projections, or rely on tree-specific structural assumptions \citep{o2026random, xu2024variance} to simplify the covariance. A generic functional abstraction providing an exact algebraic identity for the variance reduction remains absent.

\subsection{Spectral Filtering in Machine Learning}
Analyzing learning algorithms via their spectral properties has precedent in Boolean functions \citep{o2014analysis} and the implicit bias of neural networks \citep{rahaman2019spectral}. In ensemble learning, \cite{scornet2015consistency} and \cite{louppe2014understanding} explored the structural properties of forests, and \cite{curth2024random} synthesized how tree ensembles act as adaptive smoothers. However, the literature has not yet formalized the combinatorial resampling process itself as a deterministic low-pass filter operating on the Hoeffding spectrum. An explicit algebraic quantification of \textbf{\textit{how subagging attenuates high-order variance components}}, specifically through isolated transfer functions for each interaction order, remains a theoretical challenge.

\section{Mathematical Framework}
\label{sec:math_framework}

In this section, we develop the rigorous foundation for exact finite-sample analysis. We model the base learner as a symmetric measurable function and represent the subagged ensemble as an infinite-order U-statistic. To isolate variance reduction from structural approximation, we fix a test point $x \in \mathcal{X}$ and study point-wise variance. We provide a comprehensive summary of global symbols and notations in \textcolor{red}{Table \ref{tab:global_notation}} in \textcolor{red}{Appendix \ref{app:sec:notation}}.

\subsection{Problem Setup and the Subagging Estimator}
Let $\mathcal{Z} = \mathcal{X} \times \mathcal{Y}$ be a measurable space. We observe a finite training dataset $\mathcal{D}_n = \{Z_1, \dots, Z_n\}$ consisting of $n$ independent and identically distributed (i.i.d.) random variables drawn from an unknown probability measure $P$ on $\mathcal{Z}$. A deterministic base learning algorithm, trained on a sample of size $k$ ($1 \le k \le n$), is defined as a measurable function $h: \mathcal{X} \times \mathcal{Z}^k \to \mathbb{R}$. We impose the following standard regularity assumption on the base learner:

\begin{assumption}[Symmetry and Square-Integrability \citep{hoeffding1992class,wager2014asymptotic}]
\label{ass:sym_sq_int}
Let $Z_1, \dots, Z_k \stackrel{i.i.d.}{\sim} P$ be a subsample of the training data taking values in the sample space $\mathcal{Z}$. For any fixed target point $x \in \mathcal{X}$, the base learner $h(x; \cdot)$ satisfies the following regularity conditions:
\begin{itemize}
    \item[(i)] \textbf{Permutation Invariance (Symmetry):} The base learner treats the training data as an exchangeable multi-set. Specifically, for any permutation $\pi$ in the symmetric group $\mathcal{S}_k$ and for all $(z_1, \dots, z_k) \in \mathcal{Z}^k$ \citep{serfling2009approximation}, it holds strictly that
    \begin{equation*}
        h(x; z_1, \dots, z_k) = h(x; z_{\pi(1)}, \dots, z_{\pi(k)}).
    \end{equation*}
    \item[(ii)] \textbf{Square-Integrability:} The function $h$ is square-integrable with respect to the product measure $P^{\otimes k}$, ensuring a finite second moment \citep{van2000asymptotic}:
    \begin{equation*}
        \mathbb{E}_{Z_{1:k} \sim P^{\otimes k}} \left[ h(x; Z_1, \dots, Z_k)^2 \right] < \infty.
    \end{equation*}
\end{itemize}
\end{assumption}

\textcolor{red}{Assumption \ref{ass:sym_sq_int}} is a standard regularity condition for analyzing ensemble methods via U-statistics \citep{hoeffding1992class,lee2019u}. Condition (i) reflects that typical base learners (e.g., decision trees) treat subsamples as exchangeable, independent of order \citep{mentch2016quantifying}. Condition (ii) ensures finite variance, which is essential for the Hoeffding-ANOVA decomposition and variance analysis in ensemble methods \citep{buhlmann2002analyzing,wager2018estimation}. For classification tasks with bounded outputs (e.g., ${0,1}$), this condition holds trivially.

\begin{remark}[Extension to Randomized and Order-Dependent Learners]
    Although \textcolor{red}{Assumption \ref{ass:sym_sq_int}} specifies a deterministic, permutation-invariant base learner, this framework can explicitly accommodate randomized or order-dependent algorithms (e.g., decision trees with random feature splits). Let $\mathcal{A}(x; Z_{1:k}, \xi, \pi)$ denote such an estimator, where $\xi \sim P_{\xi}$ captures the internal algorithmic randomness and $\pi \sim \mathrm{Unif}(\mathcal{S}_k)$ represents a random permutation of the training data indices, both drawn independently of the training sample $Z_{1:k}$. We construct a symmetric equivalent by defining the marginalized base learner as its expectation conditional on the data: $h(x; Z_{1:k}) \triangleq \mathbb{E}_{\xi, \pi}[\mathcal{A}(x; Z_{1:k}, \xi, \pi)]$. By the law of total variance, the prediction variance of the ensemble additively decomposes into the variance of this symmetric structural component $h$ and the expected conditional variance induced by the local algorithmic noise \citep{wager2018estimation}. Therefore, analyzing the marginalized symmetric function $h$ strictly isolates the variance reduction mechanism attributable to combinatorial resampling, without compromising theoretical rigor.
\end{remark}

\textcolor{red}{Assumption \ref{ass:sym_sq_int}} holds for most unweighted ensemble base learners, such as standard decision trees and nearest neighbors, provided the algorithm does not depend on the specific ordering of the training data \citep{cover1967nearest,loh2011classification}. We analyze the \textit{Subagging} estimator. For a given subsample size $k \le n$, the subagged estimator $\hat{F}_{n,k}(x)$ aggregates the predictions of base learners trained on all possible subsets of size $k$ drawn without replacement from $\mathcal{D}_n$ \citep{mentch2016quantifying,wager2018estimation}:
\begin{equation}
\label{eq:subagging_def}
\hat{F}_{n,k}(x) = \binom{n}{k}^{-1} \sum_{S \in \binom{[n]}{k}} h(x; Z_S),
\end{equation}
where $[n] = \{1, 2, \dots, n\}$ and $\binom{[n]}{k}$ denotes the set of all $k$-combinations of $[n]$. $Z_S$ represents the subset of random variables indexed by $S$. 

\begin{remark}[Bootstrap vs. Subagging]
    We focus on sampling \textit{without replacement} (subagging). Sampling with replacement (bootstrap) introduces a positive probability of identical samples appearing multiple times within $Z_S$ (diagonal collisions) \citep{efron1994introduction}. This dependency complicates or breaks the exact orthogonality structure underlying Hoeffding decompositions required for an exact finite-sample algebraic decomposition. In the regime where $k \ll n$, subagging and bagging are asymptotically contiguous \citep{buhlmann2002analyzing,wager2014asymptotic}, yet subagging preserves mathematical exactness in finite samples.
\end{remark}

\subsection{The Hoeffding-ANOVA Decomposition and Spectral Components}
Eq. \eqref{eq:subagging_def} reveals that $\hat{F}_{n,k}(x)$ is formulated as a $U$-statistic of order $k$ with the base learner $h$ as its kernel. To decouple its variance, we invoke the Hoeffding-ANOVA decomposition \citep{hoeffding1992class,efron1981jackknife}. We project the base learner $h$ into a sum of mutually orthogonal functions of increasing interaction orders.

\begin{definition}[Canonical ANOVA Projections \citep{efron1981jackknife}]
\label{def:anova_projections}
Let $\theta(x) = \mathbb{E}_{Z_{1:k} \sim P^{\otimes k}}[h(x; Z_1, \dots, Z_k)]$ denote the expected prediction of the base learner, representing the $0$-th order baseline component. For any $1 \le c \le k$, the $c$-th order canonical projection $h_c$ is recursively defined as:
\begin{equation}
\label{eq:canonical_projections}
h_c(x; z_1, \dots, z_c) = \mathbb{E} \big[ h(x; Z_{1:k}) \mid Z_1=z_1, \dots, Z_c=z_c \big] - \sum_{j=1}^{c-1} \sum_{J \in \binom{[c]}{j}} h_j(x; z_J) - \theta(x),
\end{equation}
where $z_J$ denotes the specific subset of variables $\{z_i : i \in J\}$, and the conditional expectation integrates out the remaining $k-c$ variables under the product measure $P^{\otimes (k-c)}$.
\end{definition}

By this recursive construction, $h_c$ isolates the $c$-th order pure interaction among the training samples, completely stripped of any lower-order additive main effects. For illustration, the first-order (main effect) and second-order (interaction) kernels are strictly formulated as:
\begin{align*}
    h_1(x; z_1) &= \mathbb{E}[h \mid Z_1=z_1] - \theta(x), \\
    h_2(x; z_1, z_2) &= \mathbb{E}[h \mid Z_1=z_1, Z_2=z_2] - h_1(x; z_1) - h_1(x; z_2) - \theta(x).
\end{align*}

As a direct consequence of \textcolor{red}{Definition \ref{def:anova_projections}}, these projections exhibit fundamental orthogonality properties, which we restate below for completeness.

\begin{lemma}[Strong Degeneracy and Orthogonality \citep{hoeffding1992class}]
\label{lem:degeneracy_orthogonality}
For any $1 \le c \le k$, the $c$-th order canonical projection $h_c$ defined in Eq.~\eqref{eq:canonical_projections} satisfies the strong degeneracy property: the expectation with respect to any single argument evaluates to zero $P^{\otimes (c-1)}$-almost surely. That is, for any fixed values $\{z_j\}_{j \neq i}$ and any $i \in \{1, \dots, c\}$,
\begin{equation*}
\mathbb{E}_{Z_i \sim P} \big[ h_c(x; z_1, \dots, z_{i-1}, Z_i, z_{i+1}, \dots, z_c) \big] = 0.
\end{equation*}
Therefore, the canonical projections associated with different index subsets are mutually orthogonal. For any two distinct subsets $C_1, C_2 \subseteq \{1, \dots, k\}$ such that $C_1 \neq C_2$, the covariance between their respective projections is exactly zero:
\begin{equation*}
\mathbb{E}_{Z_{1:k} \sim P^{\otimes k}} \big[ h_{|C_1|}(x; Z_{C_1}) h_{|C_2|}(x; Z_{C_2}) \big] = 0.
\end{equation*}
\end{lemma}

\begin{proof}[Proof Sketch]
    Rather than relying on induction over interaction orders, the proof follows from a direct non-recursive representation of the projections. By applying M\"obius inversion on the Boolean lattice of subset inclusions, the recursive definition is mathematically unrolled into an explicit alternating sum of marginal expectations. Strong degeneracy is then established analytically by pairing these subsets based on the inclusion or exclusion of a target integration variable $Z_i$. By the law of total expectation, each paired subsets yield identical marginal terms but with strictly opposite signs, resulting in exact algebraic cancellation. Mutual orthogonality subsequently follows via standard sequential conditioning. Please see \textcolor{red}{Appendix~\ref{app:proof_lem_degeneracy}} for details.
\end{proof}

\textcolor{red}{Lemma \ref{lem:degeneracy_orthogonality}} provides the fundamental mathematical mechanism that allows us to decompose the exact variance of the ensemble into a finite sum of independent variance components without cross-terms. We now quantify the magnitude of these orthogonal components. 

\begin{definition}[ANOVA / Hoeffding Variance Component \citep{efron1981jackknife,xu2024variance}]
\label{def:hoeffding_variance}
For $1 \le c \le k$, the $c$-th order variance component (also known as the \textit{Hoeffding spectrum} or \textit{Sobol index}) of the base learner at target point $x$ is defined as:
\begin{equation*}
\zeta_c(x) = \text{Var}\big(h_c(x; Z_1, \dots, Z_c)\big) \equiv \mathbb{E}_{Z_{1:c} \sim P^{\otimes c}} \big[ h_c(x; Z_1, \dots, Z_c)^2 \big].
\end{equation*}
\end{definition}

\begin{remark}[Physical Interpretation of the Spectrum]
    Physically, $\zeta_c(x)$ isolates and quantifies the amount of predictive variance contributed \textit{purely} by the $c$-way interaction among the training samples. It measures how much the base learner's prediction fluctuates due to the synergistic effect of exactly $c$ data points being simultaneously present in the training set, explicitly stripping away their lower-order, independent additive contributions.
\end{remark}

Because \textcolor{red}{Lemma \ref{lem:degeneracy_orthogonality}} guarantees the strict mutual orthogonality of all canonical projections, the cross-covariance terms vanish completely. Therefore, the total variance of the base learner evaluates exactly to a weighted sum of its Hoeffding components. We formalize this fundamental property as the following lemma.

\begin{lemma}[Exact Base Variance Decomposition \citep{hoeffding1992class}]
\label{lem:base_variance_decomp}
Under \textcolor{red}{Assumption \ref{ass:sym_sq_int}}, the total variance of a single base learner trained on $k$ samples is decoupled into the finite sum:
\begin{equation}
\label{eq:base_variance}
\text{Var}_{Z_{1:k} \sim P^{\otimes k}}\big(h(x; Z_{1:k})\big) = \sum_{c=1}^k \binom{k}{c} \zeta_c(x),
\end{equation}
where the combinatorial weight $\binom{k}{c}$ accounts for the number of distinct $c$-element subsets drawn from the $k$ training samples.
\end{lemma}

\begin{proof}[Proof Sketch]
    The decomposition follows from the orthogonal structure of the canonical projections. By expanding the centralized base learner as a sum of its Hoeffding components over all subsets of $[k]$, we evaluate the variance by squaring this expansion. By \textcolor{red}{Lemma~\ref{lem:degeneracy_orthogonality}}, all cross-covariance terms between distinct subsets vanish. The remaining diagonal terms are then grouped into equivalence classes based on the subset cardinality $c$. By invoking permutation invariance in \textcolor{red}{Assumption~\ref{ass:sym_sq_int}}, the variance of any $c$-order projection is identically $\zeta_c(x)$, and the combinatorial multiplicity of choosing $c$ elements from $k$ training samples yields the binomial coefficient $\binom{k}{c}$. Please see \textcolor{red}{Appendix~\ref{app:proof_lem_base_var}} for details.
\end{proof}

Eq. \eqref{eq:base_variance} explicitly characterizes how the functional complexity of the base learning algorithm determines the distribution of its variance spectrum under $P$. Theoretical analyses of tree-based ensembles \citep{mentch2016quantifying, scornet2015consistency} indicate that structural hyperparameters directly modulate this spectrum. A structurally constrained algorithm (e.g., a shallow decision tree or a linear regularizer) processes training data predominantly through additive or low-order effects. Therefore, its variance spectrum is highly concentrated on the first-order main effect, with higher-order components decaying rapidly or being zero \citep{hastie1987generalized, hoeffding1992class}. Conversely, a high-capacity model, such as a fully grown unpruned decision tree, acts as a non-smooth local interpolator \citep{lin2006random}. Its point-wise prediction depends on the joint occurrence of multiple training samples falling into specific localized regions. This strong dependence on multi-sample interactions distributes the variance across the spectrum, yielding significant contributions from high-order interaction terms $\zeta_c(x)$ as $c$ increases \citep{hooker2007generalized}. Our primary theoretical objective is to establish exactly how the subagging estimator $\hat{F}_{n,k}$ transforms this underlying variance spectrum.

\section{Exact Variance Decomposition and Spectral Filtering}
\label{sec:main_results}

In this section, we present the exact finite-sample variance decomposition of the subagged estimator $\hat{F}_{n,k}(x)$. We then formalize the mechanism by which subagging reduces variance as a deterministic spectral filtering process over the base learner's complexity spectrum.

\subsection{Exact Finite-Sample Variance Identity}

The variance of the subagged ensemble $\hat{F}_{n,k}(x)$, as defined in Eq.~\eqref{eq:subagging_def}, depends fundamentally on the covariance between base learners trained on different subsamples. Let $S_1, S_2$ be two subsamples of size $k$ drawn uniformly at random without replacement from the training index set $[n]$. The covariance $\text{Cov}(h(x; Z_{S_1}), h(x; Z_{S_2}))$ is uniquely determined by the cardinality of their intersection, $\rho = |S_1 \cap S_2|$. 

While the variance of such estimators is governed by incomplete U-statistic theory, deriving tractable, exact finite-sample algebraic identities remains a persistent analytical challenge. Consequently, foundational works \citep{mentch2016quantifying, wager2018estimation} typically bypass finite-sample combinatorics, relying instead on asymptotic limits ($n \to \infty$) or distribution-free variance bounds. 

This difficulty arises from the conventional macroscopic evaluation of covariance. Formulating the variance through the macroscopic subsample overlap $\rho = |S_1 \cap S_2|$, which intrinsically follows a hypergeometric distribution, and taking the expectation of the infinite-order Hoeffding expansion yields a nested double summation:
\begin{equation}
    \text{Var}(\hat{F}_{n,k}(x)) = \sum_{\rho=0}^{k} \frac{\binom{k}{\rho}\binom{n-k}{k-\rho}}{\binom{n}{k}} \left[ \sum_{c=1}^{\rho} \frac{\binom{\rho}{c}}{\binom{k}{c}} \zeta_c(x) \right].
    \label{eq:hypergeometric_bottleneck}
\end{equation}
Analytically reducing this hypergeometric-binomial sum to isolate closed-form weights for an arbitrary interaction order $c$, without collapsing high-order terms via asymptotic projection, is algebraically prohibitive.

To bypass this combinatorial complexity, we shift from a macroscopic to a microscopic perspective, focusing directly on the inclusion probabilities of individual interaction subsets. By applying a probabilistic indicator decoupling argument, we leverage the unconditional independence of the algorithmic resampling. This approach cleanly disentangles the hypergeometric dependencies, entirely avoiding the nested summations in Eq.~\eqref{eq:hypergeometric_bottleneck}. Consequently, we establish that subagging acts exactly as a deterministic, finite-sample low-pass spectral filter on the Hoeffding components of any symmetric base learner.

\begin{theorem}[Exact Variance Decomposition of Subagging]
\label{thm:exact_variance}
Under \textcolor{red}{Assumption \ref{ass:sym_sq_int}}, for any finite sample size $n \ge 1$, subsample size $k \le n$, and an arbitrary test point $x \in \mathcal{X}$, the variance of the subagged estimator $\hat{F}_{n,k}(x)$ in Eq.~\eqref{eq:subagging_def} is exactly decoupled into:
\begin{equation}
\label{eq:exact_variance}
\text{Var}(\hat{F}_{n,k}(x)) = \sum_{c=1}^k \gamma_c(n,k) \binom{k}{c} \zeta_c(x),
\end{equation}
where $\zeta_c(x)$ is the $c$-th order Hoeffding variance component, and the attenuation factor $\gamma_c(n,k)$ for the $c$-th order spectrum is defined as:
\begin{equation}
\label{eq:attenuation_factor}
\gamma_c(n,k) = \frac{\binom{n-k}{k-c}}{\binom{n}{k}} = \prod_{j=0}^{c-1} \frac{k-j}{n-j}.
\end{equation}
\end{theorem}

\begin{proof}[Proof Sketch and Methodological Insight]
    Traditional variance decompositions for subagging rely on macroscopic subsample overlap (e.g., $|S_1 \cap S_2|$), leading to hypergeometric distributions and complex combinatorial summations \citep{buhlmann2002analyzing,serfling2009approximation,wager2018estimation,xu2024variance}. We bypass this bottleneck by shifting to a microscopic perspective that targets individual interaction subsets $C$. The key step is an indicator decoupling: expressing joint inclusion as $\mathbb{I}_{\{C \subseteq S_1\}} \mathbb{I}_{\{C \subseteq S_2\}}$, which isolates $S_1$ and $S_2$. Using their independence and linearity of expectation, the expectation factorizes as $(\mathbb{P}_{\mathcal{U}}(C \subseteq S))^2$.  Paired with the strict mutual orthogonality of Hoeffding-ANOVA projections, all cross-terms vanish, and the expectation factors directly into $(\mathbb{P}_{\mathcal{U}}(C \subseteq S))^2$. This structural decoupling completely avoids hypergeometric dependencies, reducing the combinatorial complexity to the computation of elementary marginal probabilities. Furthermore, because this formulation abstracts away from macroscopic overlap cardinalities, it naturally extends to other schemes (e.g., Poisson or weighted sampling) simply by substituting the uniform measure $\mathcal{U}$. Please refer to \textcolor{red}{Appendix \ref{app:proof_thm_exa_var}} for details.
\end{proof}

\begin{remark}[The Asymptotic Truncation of High-Order Variance]
    While classical asymptotic analyses establish foundational distributional limits, they rely on first-order approximations that relegate higher-order variance components to remainder terms. To achieve asymptotic normality or consistent variance estimation, existing frameworks typically invoke the Hájek projection to isolate the first-order main effect ($c=1$). For instance, \textcolor{blue}{Theorem 1} in \citep{mentch2016quantifying} and the ratio consistency results in \citep[\textcolor{blue}{Corollary 4.6}]{xu2024variance} require a subsample growth rate of $k = o(\sqrt{n})$. Even under regimes allowing faster growth, \citep[\textcolor{blue}{Lemma 7}]{wager2018estimation} analytically bounds the projection error, implying that the variance contribution of all higher-order interactions ($c \ge 2$) vanishes relative to the main effect at a rate of $\mathcal{O}(k^2/n)$. Consequently, these asymptotic regimes effectively truncate the Hoeffding spectrum, implicitly modeling the ensemble as an additive smoother. However, high-capacity interpolators memorize training data through complex multi-sample interactions, distributing substantial variance mass across high-order components ($c \gg 1$) \citep{wyner2017explaining}. This reveals an inherent analytical limitation: first-order asymptotic limits necessarily obscure the high-order structural dependencies where interpolation noise is concentrated. Our exact finite-sample identity addresses this gap by preserving the complete Hoeffding spectrum. While previous asymptotic frameworks bound higher-order interactions as negligible remainders, our results provide the algebraic resolution required to formalize a deterministic spectral filtering mechanism: subagging suppresses interpolation noise by subjecting the $c$-th order interaction component to a geometric attenuation factor $\gamma_c(n,k)$, a structural property that is effectively eliminated under standard first-order asymptotic projections.
\end{remark}

\begin{remark}[The Algebraic Identity]
    \textcolor{red}{Theorem \ref{thm:exact_variance}} establishes a deterministic algebraic identity rather than a statistical approximation. It holds irrespective of the underlying data generating process $P$ or the specific geometry of the base learner, provided the symmetry condition holds. Furthermore, it demonstrates that subagging does not reduce variance uniformly across the model's functional components. By comparing Eq.~\eqref{eq:exact_variance} with the standalone base learner's variance $\text{Var}(h) = \sum_{c=1}^k \binom{k}{c} \zeta_c(x)$, we observe that the subagging operator precisely rescales each $c$-th order interaction component by the factor $\gamma_c(n,k)$.
\end{remark}

\subsection{The Spectral Filtering Law}

To uncover the physical intuition behind the attenuation factor $\gamma_c(n,k)$, it is instructive to examine its behavior in the large-sample regime where the subsampling ratio remains constant. Let $\alpha = k/n \in (0, 1]$ denote the subsampling ratio. The following theorem characterizes $\gamma_c(n,k)$ as a geometric weight that governs the variance reduction across different interaction orders.

\begin{theorem}[The Spectral Filtering Law]
\label{thm:spectral_filtering}
For any finite $n \ge k \ge 1$ and any interaction order $1 \le c \le k$, the attenuation factor $\gamma_c(n,k)$ defined in Eq.~\eqref{eq:attenuation_factor} satisfies the following properties:
\begin{itemize}
    \item The attenuation is strictly bounded by the $c$-th power of the subsampling ratio:
    \begin{equation}
    \label{eq:gamma_bound}
        \gamma_c(n, k) = \prod_{j=0}^{c-1} \frac{k-j}{n-j} \le \left( \frac{k}{n} \right)^c = \alpha^c,
    \end{equation}
    where the equality holds if and only if $c=1$ or $k=n$.
    \item For a fixed interaction order $c$, as $n \to \infty$ with $k/n \to \alpha \in (0, 1]$, the attenuation factor converges to the geometric limit:
    \begin{equation}
    \label{eq:gamma_limit}
        \lim_{n \to \infty, \frac{k}{n} \to \alpha} \gamma_c(n, k) = \alpha^c.
    \end{equation}
\end{itemize}
Consequently, the exact ensemble variance in Eq.~\eqref{eq:exact_variance} is governed by this geometric attenuation across the Hoeffding spectrum.
\end{theorem}

\begin{proof}[Proof Sketch and Methodological Insight]
    Classical proofs rely on heavy combinatorial algebra or complex ANOVA decompositions \citep{wager2018estimation,xu2024variance}, which often obscure the resampling mechanism and complicate higher-order interactions. Instead, we adopt a probabilistic approach. We interpret the attenuation factor $\gamma_c(n,k)$ as the joint inclusion probability of a $c$-order subset under uniform sampling without replacement, and decompose it via the chain rule into conditional probabilities $\prod_{j=0}^{c-1} \mathbb{P}(E_{j+1}\mid E_1,\dots,E_j)$. This reveals that the geometric bound in Eq.~\eqref{eq:gamma_bound} arises from the \textbf{negative association} of finite-population sampling: each inclusion reduces the remaining pool, making subsequent probabilities smaller than $\alpha$. Moreover, as $n \to \infty$, this dependence vanishes for fixed $c$, yielding the geometric limit in Eq.~\eqref{eq:gamma_limit}. Please see \textcolor{red}{Appendix \ref{app:proof_thm_spe_fil}} for details.
\end{proof}

\begin{remark}[Subagging as a Low-Pass Spectral Filter]
    \textcolor{red}{Theorem \ref{thm:spectral_filtering}} provides subagging as a deterministic \textit{low-pass filter} operating on the base learner's complexity spectrum. The attenuation factor $\gamma_c(n,k)$ acts as the transfer function of this filter:
    \begin{itemize}
        \item \textbf{Preservation of Low-Order Signals ($c$ small):} For main effects ($c=1$), the variance is scaled linearly by $\alpha$. As $n$ increases, $\gamma_c(n,k)$ maintains the structural signal by decaying slowly at order $O(\alpha^c)$.
        
        \item \textbf{Annihilation of High-Order Noise ($c$ large):} For complex multi-sample interactions ($c \gg 1$), $\gamma_c(n,k)$ inflicts an exponential decay. Importantly, as $c$ approaches $k$, the property $\gamma_c(n, k) \le \alpha^c$ ensures that high-frequency overfitting noise is suppressed even more severely than the geometric rate $\alpha^c$. For interpolating models where variance is concentrated in extreme orders (e.g., $c \propto k$), the attenuation factor $\gamma_c(n,k)$ vanishes towards zero, effectively neutralizing the model's idiosyncratic memorization of the training sample.
    \end{itemize}
\end{remark}

This spectral perspective clarifies that subagging does not reduce variance uniformly. Instead, it exploits the combinatorial structure of resampling to selectively filter out high-order interaction components characteristic of over-parameterized models, while preserving the robust, low-order structural regularities. The subsampling ratio $\alpha$ dictates the \textit{cutoff frequency} of this filter. To counteract the combinatorial explosion of the binomial coefficient $\binom{k}{c}$ in high-capacity models, $\alpha$ must be sufficiently small to ensure that the geometric penalty $\alpha^c$ dominates the variance summation in Eq.~\eqref{eq:exact_variance}.

\section{Finite-Sample Bias-Variance Trade-off and Optimal Regularization}
\label{sec:bias_variance}

In this section, we integrate the inherent bias of the base learner to evaluate the finite-sample Mean Squared Error (MSE). Our objective is to analytically characterize the optimal subsampling ratio $\alpha^* = k/n$. By explicitly linking $\alpha^*$ to the base learner's functional complexity, we demonstrate why widely adopted default resampling ratios (e.g., $\alpha \approx 0.632$, derived from asymptotic bootstrap inclusion probabilities \citep{efron1983estimating, buhlmann2002analyzing}) can be structurally suboptimal when applied across diverse model capacities.

\subsection{The Bias Component of Subagging}

Let $f^*(x) = \mathbb{E}[Y \mid X=x]$ denote the true target function. The bias of the subagged estimator is determined by the expected prediction of the base learner trained on a subset of size $k$. By the linearity of expectation:
\begin{align}
\text{Bias}(\hat{F}_{n,k}(x)) &= \mathbb{E}_{Z_{1:n}}[\hat{F}_{n,k}(x)] - f^*(x) \nonumber \\
&= \binom{n}{k}^{-1} \sum_{S \in \binom{[n]}{k}} \mathbb{E}_{Z_S}[h(x; Z_S)] - f^*(x) \nonumber \\
&= \theta_k(x) - f^*(x) \triangleq \mathcal{B}(k), \label{eq:exact_bias}
\end{align}
where $\theta_k(x) = \mathbb{E}[h(x; Z_{1:k})]$. Eq.~\eqref{eq:exact_bias} confirms that subagging does not alter the inherent bias of the base learner trained on $k$ samples. To analyze the trade-off, we formalize the standard non-parametric decay rate of the base learner's bias.

\begin{assumption}[Monotonic Bias Decay]
\label{ass:mon_bias_decay}
For a fixed evaluation point $x$, the squared bias of the base learner monotonically decreases with the training sample size $k$ at a rate bounded by:
\begin{equation*}
    \mathcal{B}(k)^2 \le B_0 k^{-2\beta},
\end{equation*}
where $B_0 > 0$ is a constant depending on the intrinsic complexity of the target function, and $\beta > 0$ governs the non-parametric learning rate. This polynomial decay profile is a standard property in non-parametric regression. For optimal estimators operating on a Hölder-smooth function with exponent $e$ in a $d$-dimensional space, minimax theory yields $\beta = e / (2e + d)$ \citep{stone1982optimal, gyorfi2002distribution}. In the specific case of random trees, \cite{biau2012analysis} explicitly establishes in \textcolor{blue}{Theorem 5} that the squared bias exhibits an algebraic decay of $\mathcal{O}(k^{-Q})$ where the exponent $Q$ depends on the number of active features. By abstracting the decay rate into a generic parameter $\beta$, our analysis accommodates various consistent base learners without requiring access to their specific structural mechanisms. Furthermore, specifying such an algebraic bound on the bias is a standard theoretical prerequisite for calibrating the optimal subsampling ratio in recent asymptotic analyses of random forests \citep{wager2018estimation}.
\end{assumption}

\subsection{Finite-Sample Generalization Error Bound}

To mathematically characterize the out-of-sample prediction performance of the subagged estimator, we evaluate its point-wise MSE. By standard statistical decision theory \citep{geman1992neural, hastie2009elements}, the MSE of any estimator additively decouples into its squared bias and variance. Building upon this foundational property, we integrate our exact variance identity with the non-parametric bias decay assumption to establish a rigorous finite-sample generalization bound.

\begin{lemma}[Finite-Sample MSE Bound]
\label{lem:mse_bound}
Suppose the bias of the base learner satisfies \textcolor{red}{Assumption \ref{ass:mon_bias_decay}}. For any finite sample size $n \ge 1$, subsample size $k \le n$, and an arbitrary test point $x \in \mathcal{X}$, the point-wise generalization error of the subagged estimator $\hat{F}_{n,k}(x)$ is upper bounded by:
\begin{equation}
\label{eq:exact_mse}
\text{MSE}(\hat{F}_{n,k}(x)) \le B_0 k^{-2\beta} + \sum_{c=1}^k \gamma_c(n,k) \binom{k}{c} \zeta_c(x).
\end{equation}
\end{lemma}

\begin{proof}[Proof Sketch]
    The proof proceeds via the standard bias-variance decomposition, explicitly separating the dual sources of randomness: the data-generating distribution $P^n$ and the uniform resampling measure $\mathcal{U}$. Because the algorithmic resampling is independent of the training data, interchanging the order of expectations reveals that the expected prediction of the subagged ensemble is identical to that of a single base learner trained on $k$ samples. Therefore, the MSE additively decouples into two tractable components: the squared structural bias of a $k$-sample base learner, and the finite-sample ensemble variance. By substituting our \textcolor{red}{Theorem \ref{thm:exact_variance}} into this decomposition, we establish the full MSE bound analytically. Please see \textcolor{red}{Appendix \ref{app:proof_lem_mse_bound}} for details.
\end{proof}

Assuming the non-parametric bias decay operates at the minimax optimal rate \citep{stone1982optimal, gyorfi2002distribution}, the theoretical upper bound in Eq. (\ref{eq:exact_mse}) closely characterizes the true generalization error. Therefore, we treat the right-hand side of Eq. (\ref{eq:exact_mse}) as a principled surrogate objective for our subsequent optimization. Importantly, Eq. \eqref{eq:exact_mse} mathematically encapsulates the fundamental tension in ensemble learning. Increasing the subsample size $k$ monotonically diminishes the bias penalty term $B_0 k^{-2\beta}$. However, it simultaneously expands the support of the Hoeffding spectrum by incorporating higher-order interaction components (i.e., $\zeta_c(x)$ for large $c$) while weakening the geometric attenuation factor $\gamma_c(n,k)$. For highly complex, over-parameterized base learners, this mechanism structurally induces severe variance inflation. This trade-off highlights the necessity of dynamically calibrating the subsampling ratio to strike an optimal balance between preserving low-order structural signals and suppressing high-frequency interpolation noise.

\subsection{The Structural Dependence of the Optimal Subsampling Path}

To derive analytic insights into the minimizer of the generalization bound, we must transition from the discrete subsample size $k \in \mathbb{Z}^+$ to a continuous subsampling ratio $\alpha = k/n \in (0, 1]$. However, directly differentiating the discrete bound in Eq.~\eqref{eq:exact_mse} is mathematically ill-posed because the summation limit $k$ inherently depends on the optimization variable. To render the continuum limit tractable and differentiable, we formalize the functional capacity of the base learner.

\begin{assumption}[Structural Interaction Capacity]
\label{ass:finite_interaction}
For a given learning algorithm, we model its intrinsic functional complexity via a fixed structural capacity parameter $M \ge 1$. Specifically, for any evaluation point $x$, its Hoeffding spectrum is truncated such that higher-order variance components vanish beyond this capacity, i.e., $\zeta_c(x) = 0$ for all $c > M$. For low-complexity smoothers (e.g., shallow trees), $M$ is a small constant ($M \ll n$). For high-capacity interpolators (e.g., deep unpruned trees), $M$ is exceptionally large and scales on the order of the dataset size ($M \sim \mathcal{O}(n)$).
\end{assumption}

\textcolor{red}{Assumption~\ref{ass:finite_interaction}} is a vital mathematical instrument. It allows us to decouple the algorithm's innate spectral complexity ($M$) from the dynamic subsampling size ($k = \alpha n$). By treating $M$ as an exogenous structural parameter, we can rigorously define a continuous optimization objective without encountering ill-posed variable-limit sums, subsequently utilizing comparative statics to evaluate how high-capacity regimes (large $M$) drive the optimal subsampling ratio.

\begin{lemma}[Continuous Surrogate MSE Envelope]
\label{lem:continuous_mse}
Suppose \textcolor{red}{Assumption~\ref{ass:mon_bias_decay}} and \textcolor{red}{Assumption~\ref{ass:finite_interaction}} hold. Let $\alpha \in [M/n, 1]$ be a continuous subsampling ratio. By bounding the attenuation factor $\gamma_c(n, \alpha n) \le \alpha^c$ according to \textcolor{red}{Theorem~\ref{thm:spectral_filtering}}, and continuously extending the discrete binomial coefficients via the Gamma ($\Gamma$) function, the finite-sample MSE is bounded by the differentiable envelope:
\begin{equation} 
\label{eq:continuous_mse}
    \mathcal{M}(\alpha; \boldsymbol{\zeta}, M) \triangleq B_0(\alpha n)^{-2\beta} + \sum_{c=1}^{M} \alpha^c \frac{\Gamma(\alpha n + 1)}{\Gamma(c + 1)\Gamma(\alpha n - c + 1)} \zeta_c(x).
\end{equation}
\end{lemma}

\begin{proof}[Proof Sketch and Methodological Insight]
    The proof establishes a differentiable MSE envelope by mapping the discrete combinatorial space of subsampling to a smooth manifold via the analytic continuation of the falling factorial. The core mathematical ingenuity lies in the \textit{regularization of the Gamma function's domain}: by enforcing the structural constraint $\alpha n \ge M$, we guarantee that the arguments of the Gamma functions in the denominator remain strictly positive. This bypasses the singular poles at non-positive integers, ensuring that the surrogate objective $\mathcal{M}(\alpha; \boldsymbol{\zeta})$ is globally $C^\infty$ and analytic. Furthermore, by substituting the complex combinatorial attenuation factors with their geometric upper bounds ($\alpha^c$) derived from \textcolor{red}{Theorem \ref{thm:spectral_filtering}}, we decouple the interaction terms into a tractable power-law form. Please see \textcolor{red}{Appendix \ref{app:proof_lem_continuous_mse}} for details.
\end{proof}

To derive analytic insights into the minimizer of the generalization bound, one standard approach is to evaluate the first-order stationary condition $\partial \mathcal{M}(\alpha; \boldsymbol{\zeta}) / \partial \alpha = 0$. By applying logarithmic differentiation to the continuous Hoeffding attenuation factors, the exact stationary equation expands to:
\begin{equation} 
\label{eq:stationary_condition}
\begin{aligned}
    \frac{\partial \mathcal{M}}{\partial \alpha} &= \frac{\partial}{\partial \alpha} \left[ B_0(\alpha n)^{-2\beta} \right] + \sum_{c=1}^{M} \zeta_c(x) \frac{\partial}{\partial \alpha} \left[ H_c(\alpha) \right] \\
    &= -2\beta B_0 n^{-2\beta} \alpha^{-2\beta - 1} + \sum_{c=1}^{M} \zeta_c(x) H_c(\alpha) \left[ \frac{\partial}{\partial \alpha} \ln H_c(\alpha) \right] \\
    &= -2\beta B_0 n^{-2\beta} \alpha^{-2\beta - 1} \\
    &\quad + \sum_{c=1}^{M} \left( \alpha^c \frac{\Gamma(\alpha n + 1)}{\Gamma(c + 1)\Gamma(\alpha n - c + 1)} \right) \left[ \frac{c}{\alpha} + \sum_{j=0}^{c-1} \frac{n}{\alpha n - j} \right] \zeta_c(x) = 0.
\end{aligned}
\end{equation}
Evaluating Eq.~\eqref{eq:stationary_condition} yields a highly non-linear transcendental equation that couples fractional bias exponents ($\alpha^{-2\beta - 1}$) with logarithmic variance derivatives. Therefore, deriving an explicit closed-form analytical solution for $\alpha^*$ is mathematically intractable. To characterize the trajectory of the optimal hyperparameter without relying on explicit root-finding, we bypass local differential analysis and instead perform a monotone comparative static analysis, formalized in the following theorem.

\begin{theorem}[Structural Shift of Optimal Regularization]
\label{thm:optimal_alpha}
Suppose two base learning algorithms, $h^{(1)}$ and $h^{(2)}$, exhibit identical bias decay rates ($B_0, \beta$) but possess different Hoeffding variance spectrum sequences, denoted by $\boldsymbol{\zeta}^{(1)} = \{\zeta^{(1)}_c\}_{c=1}^M$ and $\boldsymbol{\zeta}^{(2)} = \{\zeta^{(2)}_c\}_{c=1}^M$. To isolate the effect of high-order complexity, suppose $h^{(2)}$ shares the same low-order structural variance as $h^{(1)}$ up to a critical order $c^* - 1$, but exhibits strictly greater high-order variance. Mathematically, $\zeta^{(2)}_c = \zeta^{(1)}_c$ for $c < c^*$, and $\zeta^{(2)}_c \ge \zeta^{(1)}_c$ for $c \ge c^*$ (with strict inequality for at least one $c \ge c^*$). Then, their optimal subsampling ratios satisfy:
\begin{equation*}
\alpha^*_{(2)} \le \alpha^*_{(1)}.
\end{equation*}
\end{theorem}

\begin{proof}[Proof Sketch and Methodological Insight]
    Characterizing the shift of the optimal subsampling ratio $\alpha^*$ via local calculus (e.g., the Implicit Function Theorem) is intractable due to the nonlinear, transcendental stationary conditions and potential boundary optima. Instead, we adopt Topkis's Theorem for monotone comparative statics \citep{topkis1978}. Topkis's Theorem is a foundational mathematical result that characterizes how optimal solutions shift in response to parameter changes by relying purely on the order-theoretic property of submodularity, entirely dispensing with the need for differentiability or interiority \citep{topkis1998}. Leveraging this theorem, we abstract the optimization landscape by defining a partial order $\succeq$ on the functional space of valid Hoeffding variance spectra $\mathcal{Z}$. By formally demonstrating that the negative continuous MSE envelope, $f(\alpha, \boldsymbol{\zeta}) = -\mathcal{M}(\alpha; \boldsymbol{\zeta})$, exhibits \textit{strictly decreasing differences} in $(\alpha, \boldsymbol{\zeta})$, we establish that the objective function is strictly submodular with respect to the continuous subsampling ratio $\alpha$ and the partially ordered complexity spectrum $\boldsymbol{\zeta}$. By Topkis’s Theorem, the optimal subsampling ratio $\alpha^*(\boldsymbol{\zeta})$ is therefore monotonically non-increasing as the spectrum shifts toward higher-order interactions. This avoids local differential analysis and establishes the shift as a global structural property of the bias–variance trade-off. Please see \textcolor{red}{Appendix \ref{app:proof_thm_opt_alpha}} for details.
\end{proof}

\begin{remark}[Critique of the Static 0.632 Rule and the Bias-Variance Trade-off]
    \textcolor{red}{Theorem~\ref{thm:optimal_alpha}} shows that a fixed resampling ratio (e.g., $\alpha \approx 0.632$) can be suboptimal across model complexities. This classical value arises from the asymptotic bootstrap exclusion probability \citep{breiman1996out,efron1997improvements} and does not reflect the bias-variance trade-off of the finite-sample. Instead, the optimal ratio $\alpha^*$ is determined by balancing nonparametric bias with geometric attenuation of the variance spectrum.
    \begin{itemize}
        \item \textbf{Bias-Dominated Regime (Low-Complexity Base Learners):} For structurally constrained models (e.g., shallow trees), the ability to form high-order interactions is limited, leading to a truncated Hoeffding spectrum with $\zeta_c(x)=0$ for large $c$ \citep{mentch2016quantifying}. In this regime, small $\alpha$ yields little variance reduction but increases nonparametric bias $\mathcal{O}((\alpha n)^{-2\beta})$. Hence, the optimal strategy favors a large subsampling ratio ($\alpha^* \to 1$) to retain sample size and control bias.
        
        \item \textbf{Variance-Dominated Regime (High-Capacity Interpolators):} For unpruned, interpolating models, bias is low but variance exhibits a heavy-tailed spectrum due to high-order sample interaction memorization \citep{wyner2017explaining,belkin2019reconciling}. \textcolor{red}{Theorem~\ref{thm:optimal_alpha}} shows that in this regime a smaller $\alpha^*$ is optimal: the increased bias is offset by $\alpha^c$, which exponentially suppresses high-order variance and mitigates overfitting.
    \end{itemize}
    Consequently, subagging acts not merely as a variance-smoothing operation, but as a parameterized regularization path. The subsampling ratio $\alpha$ must be dynamically calibrated to the intrinsic spectral complexity of the specific learning algorithm.
\end{remark}

\begin{corollary}[Strict Monotonicity for Interior Solutions]
\label{cor:strict_monotonicity}
Suppose the conditions of \textcolor{red}{Theorem~\ref{thm:optimal_alpha}} hold. If the optimal subsampling ratios are interior solutions, i.e., $\alpha^*_{(1)}, \alpha^*_{(2)} \in (0, 1)$, then the inequality is strict:
\begin{equation*}
\alpha^*_{(2)} < \alpha^*_{(1)}.
\end{equation*}
\end{corollary}

\begin{proof}
    \textcolor{red}{Theorem \ref{thm:optimal_alpha}} establishes the weak inequality $\alpha^*_{(2)} \le \alpha^*_{(1)}$. The continuous objective $f(\alpha, \boldsymbol{\zeta}) = -\mathcal{M}(\alpha; \boldsymbol{\zeta})$ exhibits strictly decreasing differences in $(\alpha, \boldsymbol{\zeta})$. According to Edlin and Shannon's Theorem \citep{edlin1998strict}, yielding the strict inequality $\alpha^*_{(2)} < \alpha^*_{(1)}$. 
\end{proof}

\textcolor{red}{Corollary~\ref{cor:strict_monotonicity}} indicates that, provided the optimal solutions do not bind at the boundaries, an increase in high-order variance components requires a smaller subsampling ratio to minimize the surrogate MSE envelope.

\section{Concrete Case Study: The Bi-Modal Spectrum Learner}
\label{sec:case_study}

In this section, we analyze a stylized case study to illustrate the implications of the spectral filtering framework. By employing a mathematically explicit proxy for high-capacity estimators, we characterize \textbf{\textit{how the maximum interaction order, representing a structural proxy for the depth of an unpruned decision tree, governs the behavior of the optimal subsampling ratio}}.

\subsection{Analytical Formulation of the Spectrum}

Consider a base learner whose functional complexity is concentrated exclusively at two extremes: a robust linear main effect ($c=1$) and a highly complex, spurious multi-sample interaction of high order ($c=M$, where $1 \ll M \le k$). This represents a \textit{pure interpolator} that captures the global signal but also memorizes high-frequency noise configurations. We formalize this by defining a \textit{Bi-modal Hoeffding Spectrum}:
\begin{equation}
\label{eq:bimodal_spectrum}
\zeta_c(x) = 
\begin{cases} 
\sigma_1^2 \binom{k}{1}^{-1}, & \text{if } c = 1, \\
\sigma_M^2 \binom{k}{M}^{-1}, & \text{if } c = M, \\
0, & \text{otherwise}.
\end{cases}
\end{equation}
Here, $\sigma_1^2 > 0$ and $\sigma_M^2 > 0$ represent the invariant total variance contributions of the signal and the overfitting noise, respectively. Therefore, the total variance of the standalone base learner trained on $k$ samples is identically constant: $\text{Var}(h(x; Z_{1:k})) = \sigma_1^2 + \sigma_M^2$, which strictly satisfies the $\mathcal{O}(1)$ interpolation variance regime.

\subsection{Exact Filtered MSE Envelope}

We apply \textcolor{red}{Theorem~\ref{thm:spectral_filtering}} and \textcolor{red}{Assumption \ref{ass:mon_bias_decay}} ($\mathcal{B}(k)^2 = B_0 k^{-2\beta}$) to this bi-modal learner. Substituting Eq.~\eqref{eq:bimodal_spectrum} into Eq.~\eqref{eq:continuous_mse}, the continuous finite-sample MSE simplifies into an explicit algebraic polynomial:
\begin{align}
    \mathcal{M}(\alpha; M) &= B_0(\alpha n)^{-2\beta} + \gamma_1(\alpha n, n)\sigma_1^2 + \gamma_M(\alpha n, n)\sigma_M^2 \nonumber \\
    &= B_0 n^{-2\beta} \alpha^{-2\beta} + \sigma_1^2 \alpha + \sigma_M^2 \alpha^M \left( 1 + \mathcal{O}\left(\frac{M^2}{n}\right) \right). \label{eq:bimodal_mse}
\end{align}
Eq.~\eqref{eq:bimodal_mse} distills the core mechanics of ensemble learning: the low-order signal variance is attenuated linearly ($\alpha$), while the high-order interpolation noise is subjected to a severe geometric penalty ($\alpha^M$).

\subsection{Closed-Form Scaling of the Optimal Subsampling Ratio}

To find the optimal regularization parameter $\alpha^*$, we differentiate $\mathcal{M}(\alpha; M)$ with respect to $\alpha$ and set it to zero:
\begin{equation} 
\label{eq:derivative_mse}
    \frac{\partial \mathcal{M}}{\partial \alpha} = -2\beta B_0 n^{-2\beta} \alpha^{-2\beta-1} + \sigma_1^2 + M \sigma_M^2 \alpha^{M-1} = 0.
\end{equation}
This equation elucidates the precise trade-off mechanism. The marginal reduction in bias (the negative term) must be balanced by the marginal increase in variance. Importantly, we evaluate two operational regimes.

\textbf{Regime 1: Low-Complexity Base Learner ($M \to 1$).} If the base learner is highly regularized (e.g., a shallow tree stump), the high-order complexity vanishes. The optimal ratio $\alpha^*$ is determined primarily by balancing the bias against the linear variance $\sigma_1^2$. The optimal $\alpha^*$ is structurally large, reflecting that smaller subsampling is unnecessary.

\textbf{Regime 2: High-Complexity Regime ($\sigma_M^2 \gg \sigma_1^2$).} Consider an interpolating base learner where the prediction variance is concentrated in high-order sample interactions ($\sigma_M^2 \gg \sigma_1^2$), representing the memorization of spurious high-frequency noise. By balancing the bias-variance trade-off and equating the dominant high-order noise derivative with the bias derivative in Eq.~(\ref{eq:derivative_mse}), we obtain the asymptotic relationship for the optimal subsampling ratio:
\begin{equation*}
    \alpha^* \sim \mathcal{O}\left( \left( \sigma_M^2 \right)^{-\frac{1}{M+2\beta}} \right).
\end{equation*}
This result demonstrates that as the variance contribution from high-order interactions diverges, the optimal subsampling ratio must asymptotically vanish ($\alpha^* \to 0$) to maintain an optimal bias-variance trade-off. 

\begin{remark}
    While the magnitude of the high-order noise $\sigma_M^2$ encourages a reduction in $\alpha^*$, the interaction depth $M$ governs the sensitivity of the regularization path. Due to the geometric nature of the attenuation factor $\alpha^M$, a large interaction order $M$ (characteristic of unpruned decision trees where $M \sim \mathcal{O}(k)$) functions as a sharp spectral filter. Mathematically, a subsampling ratio even moderately below unity (e.g., $\alpha = 0.632$ or $0.8$) results in the substantial suppression of high-order variance components (e.g., $0.63^{100} \approx 10^{-20}$). This provides a theoretical grounding for the empirical effectiveness of classical heuristic ratios in the context of Random Forests \citep{breiman2001random}. The rapid decay of the Hoeffding filter at high orders suggests that a moderate subsampling ratio can effectively attenuate interpolation noise without necessitating the substantial reduction of $\alpha$ that would otherwise lead to a prohibitive increase in bias.
\end{remark}

\section{Algorithm Design and Theoretical Verification}
\label{sec:exp_res}

In this section, we operationalize our theoretical framework to empirically investigate the structural shift of the optimal subsampling ratio $\alpha^*$. We introduce Complexity-Guided Adaptive Subsampling (CGAS), a theoretically principled algorithm that adapts the insights of \textcolor{red}{Theorem \ref{thm:spectral_filtering}} into practice. CGAS algorithm uses the base learner's structural capacity as a proxy to constrain the search space for $\alpha$. To examine the comparative statics established in \textcolor{red}{Theorem \ref{thm:optimal_alpha}}, we design a controlled evaluation spanning two functional complexity regimes: low-complexity smoothers and high-complexity interpolators. Across diverse benchmark datasets, we evaluate the generalization performance of the dynamically calibrated $\alpha^*$ against standard fixed-rate resampling heuristics.

\begin{algorithm}[ht]
\caption{Complexity-Guided Adaptive Subsampling (CGAS)}
\label{alg:cgas}
\begin{algorithmic}[1]
\STATE \textbf{Input:} Training data $\mathcal{D}_n$, base learning algorithm $h$ with maximum complexity proxy $d_{max}$ (e.g., max\_depth), pilot ensemble size $B_{search}$ (e.g., 30), final ensemble size $B_{final}$ (e.g., 100), number of CV folds $K$ (e.g., 3).
\STATE \textbf{Phase 1: Theoretical Grid Restriction}
\IF{$d_{max}$ is unconstrained or $d_{max} > 10$} 
    \STATE \COMMENT{High-Complexity Regime: Severe high-order interpolation noise detected.}
    \STATE \COMMENT{Necessitates stringent geometric filtering to annihilate noise.}
    \STATE Initialize search grid: $\mathcal{A} \leftarrow \{0.1, 0.2, 0.3, 0.4\}$.
\ELSE
    \STATE \COMMENT{Low-Complexity Regime: Variance concentrated in lower-order structural signals.}
    \STATE \COMMENT{Requires minimal variance penalty to prevent bias explosion.}
    \STATE Initialize search grid: $\mathcal{A} \leftarrow \{0.6, 0.7, 0.8, 0.9, 0.95\}$.
\ENDIF
\STATE \textbf{Phase 2: Internal $K$-Fold Cross-Validation}
\STATE Initialize $\alpha^* \leftarrow 0.632$ and $MSE_{best} \leftarrow \infty$.
\STATE Split $\mathcal{D}_n$ into $K$ strictly disjoint folds.
\FOR{each $\alpha \in \mathcal{A}$}
    \FOR{$k = 1$ to $K$}
        \STATE Train pilot subagged ensemble (size $B_{search}$) on training folds using ratio $\alpha$ without replacement.
        \STATE Compute Mean Squared Error ($MSE_k$) on the validation fold.
    \ENDFOR
    \STATE Compute $\overline{MSE}_\alpha \leftarrow \frac{1}{K} \sum_{k=1}^K MSE_k$.
    \IF{$\overline{MSE}_\alpha < MSE_{best}$}
        \STATE $MSE_{best} \leftarrow \overline{MSE}_\alpha$
        \STATE $\alpha^* \leftarrow \alpha$
    \ENDIF
\ENDFOR
\STATE \textbf{Phase 3: Final Ensemble Training}
\STATE Train the final subagged ensemble of size $B_{final}$ on the entire dataset $\mathcal{D}_n$ using the optimal ratio $\alpha^*$.
\STATE \textbf{Output:} The trained final ensemble and the optimal regularization parameter $\alpha^*$.
\end{algorithmic}
\end{algorithm}

\subsection{Bridging Theory and Practice: A Complexity-Guided Adaptive Heuristic}

While \textcolor{red}{Theorem \ref{thm:optimal_alpha}} characterizes the structural shift in optimal regularization, determining the exact optimizer $\alpha^{*}$ is analytically and computationally intractable. This computation necessitates access to the unknown nonparametric bias derivative and the complete Hoeffding variance spectrum $\{\zeta_c\}_{c=1}^k$. Although recent advances in infinite-order U-statistics \citep{wager2018estimation, xu2024variance} permit the empirical estimation of low-order variance components via the infinitesimal jackknife, recovering the entire spectrum incurs prohibitive computational complexity. To operationalize the insights from \textcolor{red}{Theorem \ref{thm:spectral_filtering}} without requiring explicit spectrum estimation, we introduce CGAS algorithm, an empirical procedure detailed in \textcolor{red}{Algorithm \ref{alg:cgas}}.

Rather than estimating the empirical spectrum, CGAS utilizes a structural hyperparameter of the base learner (e.g., maximum tree depth $d_{\max}$) as a proxy for the effective maximum interaction order. By construction, heavily regularized models (e.g., $d_{\max} \le 10$) constrain the capacity to form multi-sample interactions ($\zeta_c \approx 0$ for large $c$), whereas unpruned interpolators distribute substantial variance mass to high-order components. Guided by the monotone comparative statics established in \textcolor{red}{Theorem \ref{thm:optimal_alpha}}, CGAS restricts the search grid for $\alpha$ based on this complexity proxy and selects an empirical optimum using internal $K$-fold cross-validation. $K$-fold CV is employed instead of Out-of-Bag (OOB) estimation to mitigate the variance instability characteristic of OOB errors when evaluating small pilot ensembles ($B_{\text{search}}$). By conditioning the search space on the structural complexity of the base learner, the CGAS algorithm avoids the suboptimality inherent in adopting fixed resampling heuristics (e.g., $\alpha \approx 0.632$) irrespective of the model's capacity. Consequently, it operationalizes the subsampling ratio as a complexity-dependent tuning parameter rather than a static variance-reduction constant.

\subsection{Experimental Design and Verification Protocol}

To empirically investigate \textbf{\textit{how the optimal subsampling ratio $\alpha^*$ varies in response to the base learner's interaction complexity}}, we formulated a controlled experimental setup. 

\paragraph{Functional Complexity Regimes} 
We designed two distinct experimental regimes corresponding to the opposite ends of the Hoeffding complexity spectrum:
\begin{itemize}
    \item \textbf{Regime 1: Low-Complexity Smoothers (Shallow Trees on Original Data).} We constructed a low-complexity regime by constraining the base learners to a maximum tree depth of $d_{\max} = 3$ and training them on the unmodified target distributions. Under these structural constraints, the models possess limited capacity to form complex multi-sample interactions, confining their Hoeffding variance spectrum primarily to low-order main effects ($\zeta_1, \zeta_2$). Therefore, high-order variance components are negligible, and the bias-variance trade-off is governed predominantly by the nonparametric bias.
    \item \textbf{Regime 2: High-Complexity Interpolators (Deep Trees with Label Noise).} To induce interpolation behavior and amplify high-order variance, we configured the base learners as fully unpruned decision trees ($d_{\max} = \infty$) and injected additive Gaussian noise (scaled to $1.5 \sigma_y$) into the target variable $y$. In this regime, the trees act as interpolators that memorize the label noise, distributing substantial variance mass across high-order interaction components.
\end{itemize}

\paragraph{Datasets}
We evaluated our theoretical framework across a diverse suite of nine public datasets sourced from OpenML, encompassing a wide range of sample sizes and feature dimensions to ensure robust empirical validation. Specifically, the benchmark included \textit{yprop\_4\_1} ($8885$ samples, $252$ features), \textit{wine\_quality} ($6497$ samples, $12$ features), \textit{superconduct} ($21263$ samples, $82$ features), \textit{dataset\_sale} ($10738$ samples, $15$ features), \textit{coli\_20} ($1440$ samples, $1025$ features), \textit{elevators} ($16599$ samples, $17$ features), \textit{sarcos} ($44484$ samples, $22$ features), \textit{california\_housing} ($20640$ samples, $10$ features), and \textit{2dplanes} ($40768$ samples, $11$ features). This diverse collection ensures that our empirical analysis tests the learning dynamics across various data scales, from standard low-dimensional regression tasks to extreme over-parameterized regimes where functional complexity heavily influences the generalization error.

\paragraph{Baselines}
To strictly ablate the effect of the combinatorial resampling mechanism, all methods were initialized with \texttt{random\_state=42} and \texttt{n\_estimators=100}, and benchmark our proposed CGAS against two Subagging variants utilizing universally adopted heuristic sampling fractions without replacement, namely Subagging ($0.632$) and Subagging ($0.8$).

\paragraph{Metrics and Evaluation Protocol}
To guarantee the statistical robustness of our comparative statics, we performed $5 \times 10$ repeated K-Fold Cross-Validation ($5$ repeats of strictly disjoint $10$-folds with varying random seeds). We reported the out-of-sample Mean Squared Error (MSE) and Mean Absolute Error (MAE). Furthermore, to rigorously assert the superiority of the dynamically calibrated $\alpha^*$, we performed a one-sided Wilcoxon Signed-Rank Test across the CV iterations, testing the null hypothesis that the baselines perform equally to or better than CGAS ($p < 0.001$ indicates statistical significance).

\paragraph{Environmental Setting}
All experiments were conducted locally using Jupyter Notebook in an Anaconda environment on Windows 10 (Version 10.0.19045). The hardware infrastructure utilized an Intel processor (Intel64 Family 6 Model 151 Stepping 2, GenuineIntel) equipped with 20 logical cores. The software stack was built on Python 3.12.4, utilizing the following key libraries: pandas 2.2.2, numpy 1.26.4, scikit-learn 1.7.2, and statsmodels 0.14.2.

\subsection{Empirical Results and Analysis}

The empirical analysis is structured to evaluate four aspects: the structural shift of $\alpha^*$, the generalization performance against fixed-rate subagging, the applicability to alternative symmetric base learner, and the efficacy of complexity-guided sampling against standard bootstrap.

\begin{figure}[htbp]
    \centering
    \includegraphics[width=0.85\textwidth]{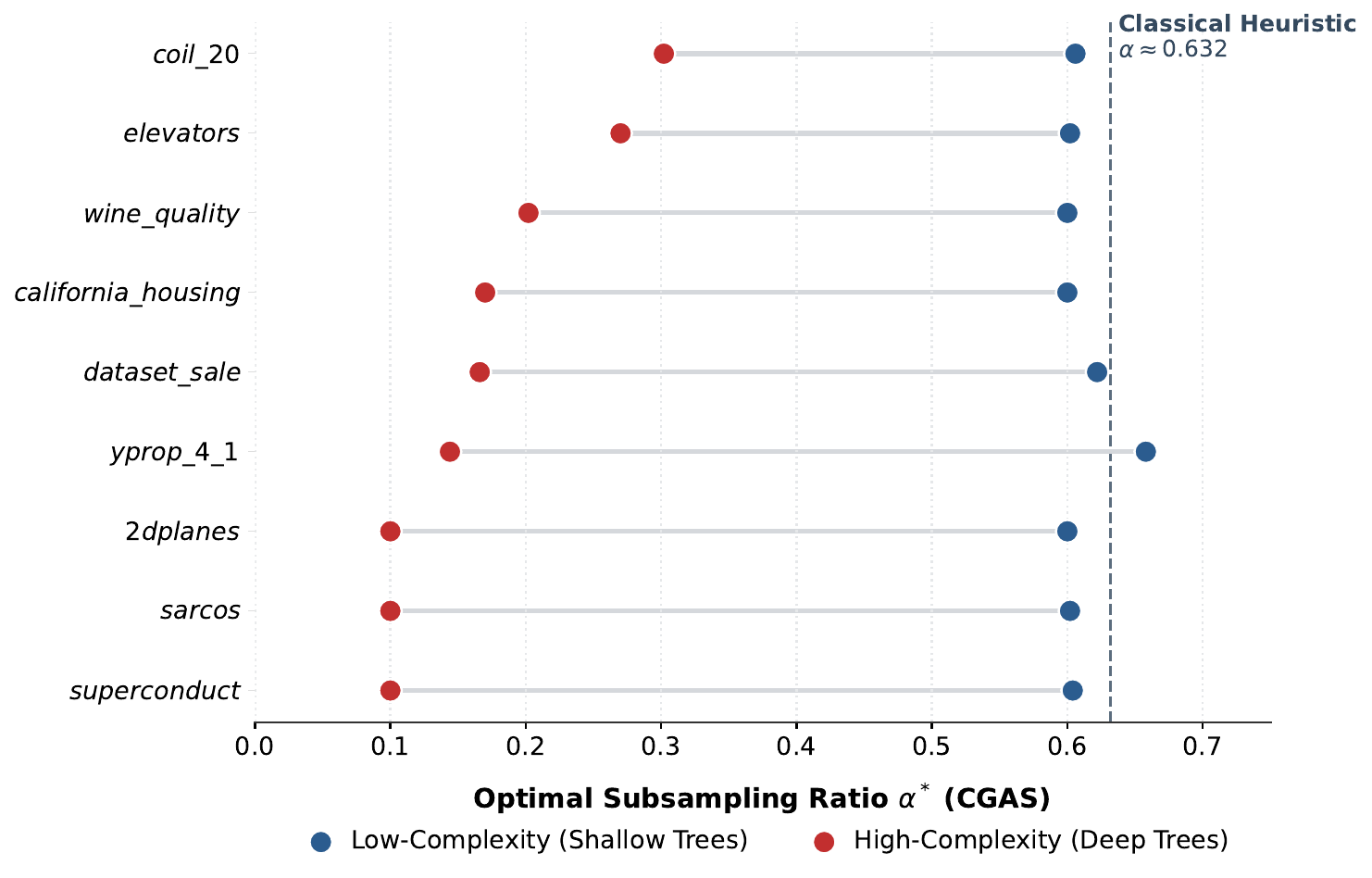}
    \vspace{-0.5em}
    \caption{Structural shift of the optimal subsampling ratio $\alpha^*$ across nine evaluated datasets. Blue markers represent the low-complexity regime (shallow trees), clustering near the classical heuristic $\alpha \approx 0.632$ (dotted line) to preserve structural signals. Red markers represent the high-complexity regime (unpruned deep trees), demonstrating a systematic shift toward smaller subsampling ratios. This empirical shift corroborates the comparative statics in \textcolor{red}{Theorem \ref{thm:optimal_alpha}}, highlighting the necessity of stringent geometric filtering for high-capacity interpolators.}
    \label{fig:alpha_shift}
\end{figure}

\subsubsection{Structural Shift of the Optimal Subsampling Ratio}
\label{sec:result_alpha_shift}

\textcolor{red}{Figure \ref{fig:alpha_shift}} shows the distribution of the optimal subsampling ratio $\alpha^*$ identified by CGAS across datasets under two complexity regimes. The results indicated that $\alpha^*$ is not fixed but shifts systematically with model complexity. In Regime 1 (low complexity), $\alpha^*$ remained close to the classical range ($0.60$-$0.65$). In Regime 2 (high complexity), it decreased significantly, concentrating within $0.1$-$0.3$. This behavior supports \textcolor{red}{Theorem \ref{thm:optimal_alpha}}. Shallow trees, with limited capacity, require larger $\alpha$ to retain low-order structure and avoid bias. In contrast, deep unpruned trees act as interpolators, capturing high-order variance; thus, a much smaller $\alpha$ is needed to strongly attenuate high-frequency noise and prevent overfitting.

\begin{table}[htbp]
\centering
\caption{Out-of-sample MSE and MAE across various datasets under two distinct functional complexity regimes. The best performance metrics in each regime are highlighted in \textbf{bold}.}
\label{tab:results}
\resizebox{\textwidth}{!}{
\begin{tabular}{cccccccc}
\toprule
\multirow{2}{*}{Datasets} & \multirow{2}{*}{Metric} & \multicolumn{3}{c}{Regime 1 (Low-Complexity)} & \multicolumn{3}{c}{Regime 2 (High-Complexity)} \\
\cmidrule(lr){3-5} \cmidrule(lr){6-8}
& & 0.632 & 0.8 & CGAS & 0.632 & 0.8 & CGAS \\
\midrule
\multirow{2}{*}{yprop\_4\_1} 
& MSE & 0.9672$^*$ & 0.9686$^*$ & \textbf{0.9670} & 1.0192$^*$ & 1.0427$^*$ & \textbf{0.9923} \\
& MAE & \textbf{0.6605} & 0.6610$^*$ & \textbf{0.6605} & 0.7851$^*$ & 0.7936$^*$ & \textbf{0.7745} \\
\midrule
\multirow{2}{*}{wine\_quality} 
& MSE & 0.7231$^*$ & 0.7276$^*$ & \textbf{0.7221} & 0.9326$^*$ & 0.9821$^*$ & \textbf{0.8843} \\
& MAE & 0.6671$^*$ & 0.6695$^*$ & \textbf{0.6666} & 0.7663$^*$ & 0.7859$^*$ & \textbf{0.7457} \\
\midrule
\multirow{2}{*}{superconduct} 
& MSE & 0.2881$^*$ & 0.2908$^*$ & \textbf{0.2876} & 0.8628$^*$ & 0.9650$^*$ & \textbf{0.7468} \\
& MAE & 0.3745$^*$ & 0.3763$^*$ & \textbf{0.3742} & 0.7398$^*$ & 0.7603$^*$ & \textbf{0.6886} \\
\midrule
\multirow{2}{*}{dataset\_sale} 
& MSE & 0.4051$^*$ & 0.4077$^*$ & \textbf{0.4043} & 0.8538$^*$ & 0.8973$^*$ & \textbf{0.8019} \\
& MAE & \textbf{0.3825} & 0.3832$^*$ & \textbf{0.3825} & 0.7350$^*$ & 0.7530$^*$ & \textbf{0.7106} \\
\midrule
\multirow{2}{*}{coil\_20} 
& MSE & 0.3602$^*$ & 0.3762$^*$ & \textbf{0.3582} & 0.7586$^*$ & 0.7826$^*$ & \textbf{0.7314} \\
& MAE & 0.4546$^*$ & 0.4586$^*$ & \textbf{0.4542} & 0.6956$^*$ & 0.7069$^*$ & \textbf{0.6839} \\
\midrule
\multirow{2}{*}{elevators} 
& MSE & 0.4809$^*$ & 0.4896$^*$ & \textbf{0.4796} & 0.7902$^*$ & 0.8078$^*$ & \textbf{0.7726} \\
& MAE & 0.4575$^*$ & 0.4614$^*$ & \textbf{0.4570} & 0.7080$^*$ & 0.7160$^*$ & \textbf{0.7002} \\
\midrule
\multirow{2}{*}{california\_housing} 
& MSE & 0.4488$^*$ & 0.4569$^*$ & \textbf{0.4472} & 0.7979$^*$ & 0.8125$^*$ & \textbf{0.7832} \\
& MAE & 0.5021$^*$ & 0.5068$^*$ & \textbf{0.5012} & 0.7107$^*$ & 0.7172$^*$ & \textbf{0.7048} \\
\midrule
\multirow{2}{*}{sarcos} 
& MSE & 0.3741$^*$ & 0.3804$^*$ & \textbf{0.3729} & 0.7796$^*$ & 0.8079$^*$ & \textbf{0.7229} \\
& MAE & 0.4573$^*$ & 0.4619$^*$ & \textbf{0.4563} & 0.7036$^*$ & 0.7157$^*$ & \textbf{0.6779} \\
\midrule
\multirow{2}{*}{2dplanes} 
& MSE & 0.1793$^*$ & 0.1831$^*$ & \textbf{0.1786} & 0.9314$^*$ & 1.0197$^*$ & \textbf{0.7387} \\
& MAE & 0.3446$^*$ & 0.3484$^*$ & \textbf{0.3438} & 0.7681$^*$ & 0.8028$^*$ & \textbf{0.6856} \\
\bottomrule
\end{tabular}
}
\textbf{Note:} All baselines marked with $*$ are significantly different from CGAS according to the Wilcoxon signed-rank test, where $^*$ denotes $p<0.001$.
\end{table}

\subsubsection{Performance Evaluation Against Fixed-Rate Subagging}
\label{sec:result_subagging}

\textcolor{red}{Table \ref{tab:results}} reports the generalization performance of different subsampling strategies without replacement, providing consistent empirical support for the theoretical mechanisms in \textcolor{red}{Theorem \ref{thm:spectral_filtering}} and \textcolor{red}{Theorem \ref{thm:optimal_alpha}}. In \textbf{Regime 2 (High-Complexity)}, unpruned deep trees exhibit interpolating behavior. Consequently, their heavy-tailed Hoeffding spectra distribute substantial variance mass to high-order interactions. Standard fixed heuristics (e.g., $\alpha \approx 0.632$ or $0.8$) provide insufficient geometric attenuation here, elevating generalization error. Guided by the complexity proxy, CGAS adaptively selected smaller subsampling ratios to suppress high-frequency noise, achieving consistently lower out-of-sample MSE and MAE across all datasets. For instance, it reduced MSE by $20.7\%$ relative to the $0.632$ baseline on the \textit{2dplanes} dataset. In \textbf{Regime 1 (Low-Complexity)}, heavily regularized shallow trees constrain variance primarily to low-order main effects. Since high-order noise is negligible, imposing stringent geometric penalties via a small $\alpha$ unnecessarily inflates nonparametric bias. CGAS correctly shifted toward larger ratios, preserving structural signals and preventing severe bias inflation. It yielded performance either statistically comparable or superior (Wilcoxon signed-rank test, $p < 0.001$) to fixed-rate baselines, defaulting safely to minimal regularization. In summary, these findings demonstrate that static resampling constants are generally suboptimal for high-capacity models. The subsampling fraction $\alpha$ must instead be operationalized as a dynamic regularization parameter, strictly calibrated to the learning algorithm's spectral complexity.

\subsubsection{Generalization to Alternative Symmetric Learners}

Since the variance identity (\textcolor{red}{Theorem \ref{thm:spectral_filtering}}) applies to any permutation-invariant base learner (\textcolor{red}{Assumption \ref{ass:sym_sq_int}}), we evaluated $K$-Nearest Neighbors (KNN) to verify that the spectral filtering mechanism operates independently of the specific algorithm. Following the established protocol, we configured KNN into two regimes: in \textbf{Regime 1 (Low-Complexity)}, a smoothed KNN ($K=20$) trained on original data restricts the Hoeffding spectrum to low-order components; in \textbf{Regime 2 (High-Complexity)}, an interpolating 1-NN ($K=1$) trained on noisy targets yields a heavy-tailed variance spectrum driven by high-order multi-sample interactions.

\textcolor{red}{Table \ref{tab:knn_results}} details the generalization performance of subagged KNN ensembles, corroborating the comparative statics of \textcolor{red}{Theorem \ref{thm:optimal_alpha}}. For high-complexity 1-NN interpolators, static heuristics ($\alpha \in \{0.632, 0.8\}$) systematically under-regularized the models. By applying a stricter geometric penalty via smaller $\alpha$, CGAS suppressed high-frequency noise and mitigated severe error inflation, reducing out-of-sample MSE by 20.7\% to 27.1\% relative to the 0.632 baseline across the datasets (e.g., from 1.1394 to 0.8308 on \textit{california\_housing}). Conversely, for low-complexity smoothed KNNs, CGAS avoided nonparametric bias inflation by adaptively favoring larger subsampling ratios to preserve low-order structural signals. It performed comparably or superiorly to fixed baselines, barring a marginal MSE exception on the \textit{wine\_quality} dataset (0.6420 vs. 0.6415 for $\alpha=0.8$). The parallel learning dynamics observed across decision trees and KNN provide empirical evidence that subagging operates as an algorithm-agnostic spectral filter, indicating that the suboptimality of fixed resampling fractions is a property of finite-sample structural complexity rather than specific model architectures.

\begin{table}[htbp]
\centering
\caption{Out-of-sample MSE and MAE for \textbf{KNN} across various datasets under two distinct functional complexity regimes. The best performance metrics in each regime are highlighted in \textbf{bold}.}
\label{tab:knn_results}
\resizebox{\textwidth}{!}{
\begin{tabular}{cccccccc}
\toprule
\multirow{2}{*}{Datasets} & \multirow{2}{*}{Metric} & \multicolumn{3}{c}{Regime 1 (Low-Complexity)} & \multicolumn{3}{c}{Regime 2 (High-Complexity)} \\
\cmidrule(lr){3-5} \cmidrule(lr){6-8}
& & 0.632 & 0.8 & CGAS & 0.632 & 0.8 & CGAS \\
\midrule
\multirow{2}{*}{yprop\_4\_1} 
& MSE & 0.9649$^*$ & 0.9660$^*$ & \textbf{0.9647} & 1.3692$^*$ & 1.5465$^*$ & \textbf{1.0295} \\
& MAE & 0.6622$^*$ & 0.6647$^*$ & \textbf{0.6618} & 0.9137$^*$ & 0.9724$^*$ & \textbf{0.7900} \\
\midrule
\multirow{2}{*}{wine\_quality} 
& MSE & 0.6449$^*$ & \textbf{0.6415} & 0.6420 & 1.1593$^*$ & 1,2971$^*$ & \textbf{0.9195} \\
& MAE & 0.6297$^*$ & \textbf{0.6267} & \textbf{0.6267} & 0.8543$^*$ & 0.9031$^*$ & \textbf{0.7608} \\
\midrule
\multirow{2}{*}{superconduct} 
& MSE & 0.1449$^*$ & 0.1360$^*$ & \textbf{0.1304} & 0.9728$^*$ & 1.0731$^*$ & \textbf{0.7637} \\
& MAE & 0.2316$^*$ & 0.2220$^*$ & \textbf{0.2157} & 0.7863$^*$ & 0.8244$^*$ & \textbf{0.6968} \\
\midrule
\multirow{2}{*}{dataset\_sale} 
& MSE & 0.5321$^*$ & 0.5273$^*$ & \textbf{0.5230} & 1.1904$^*$ & 1.3515$^*$ & \textbf{0.8832} \\
& MAE & 0.4502$^*$ & 0.4484$^*$ & \textbf{0.4475} & 0.8658$^*$ & 0.9240$^*$ & \textbf{0.7404} \\
\midrule
\multirow{2}{*}{coil\_20} 
& MSE & 0.3312$^*$ & 0.2932$^*$ & \textbf{0.2744} & 0.9872$^*$ & 1.1251$^*$ & \textbf{0.7569} \\
& MAE & 0.3309$^*$ & 0.2928$^*$ & \textbf{0.2717} & 0.7952$^*$ & 0.8496$^*$ & \textbf{0.6946} \\
\midrule
\multirow{2}{*}{elevators} 
& MSE & 0.3454$^*$ & 0.3373$^*$ & \textbf{0.3321} & 1.1401$^*$ & 1.2973$^*$ & \textbf{0.8376} \\
& MAE & 0.3715$^*$ & 0.3679$^*$ & \textbf{0.3657} & 0.8478$^*$ & 0.9049$^*$ & \textbf{0.7261} \\
\midrule
\multirow{2}{*}{california\_housing} 
& MSE & 0.3103$^*$ & 0.3059$^*$ & \textbf{0.3036} & 1.1394$^*$ & 1.2981$^*$ & \textbf{0.8308} \\
& MAE & 0.3862$^*$ & 0.3820$^*$ & \textbf{0.3796} & 0.8484$^*$ & 0.9055$^*$ & \textbf{0.7243} \\
\midrule
\multirow{2}{*}{sarcos} 
& MSE & 0.0597$^*$ & 0.0511$^*$ & \textbf{0.0460} & 1.0263$^*$ & 1.1686$^*$ & \textbf{0.7502} \\
& MAE & 0.1634$^*$ & 0.1505$^*$ & \textbf{0.1425} & 0.78072$^*$ & 0.8613$^*$ & \textbf{0.6914} \\
\midrule
\multirow{2}{*}{2dplanes} 
& MSE & 0.0713$^*$ & 0.0729$^*$ & \textbf{0.0710} & 1.0173$^*$ & 1.1257$^*$ & \textbf{0.7579} \\
& MAE & 0.2137$^*$ & 0.2160$^*$ & \textbf{0.2132} & 0.8045$^*$ & 0.8451$^*$ & \textbf{0.6944} \\
\bottomrule
\end{tabular}
}
\textbf{Note:} All baselines marked with $*$ are significantly different from CGAS according to the Wilcoxon signed-rank test, where $^*$ denotes $p<0.001$.
\end{table}

\subsubsection{Efficacy Against Standard Bootstrap (Random Forests)}
\label{sec:result_bootstrap}

To evaluate the practical implications of our framework, we conducted a controlled comparison between the standard Random Forest ($\text{RF}$) \citep{breiman2001random} and a complexity-calibrated variant ($\text{RF}^*$). The baseline $\text{RF}$ employs default bootstrap sampling ($n$ draws with replacement), whereas $\text{RF}^*$ preserves the bootstrap mechanism but constrains the subsample size via \texttt{max\_samples=}$\alpha^*$, where $\alpha^*$ is determined by the CGAS algorithm.

\begin{figure}[htbp]
    \centering
    \includegraphics[width=\textwidth]{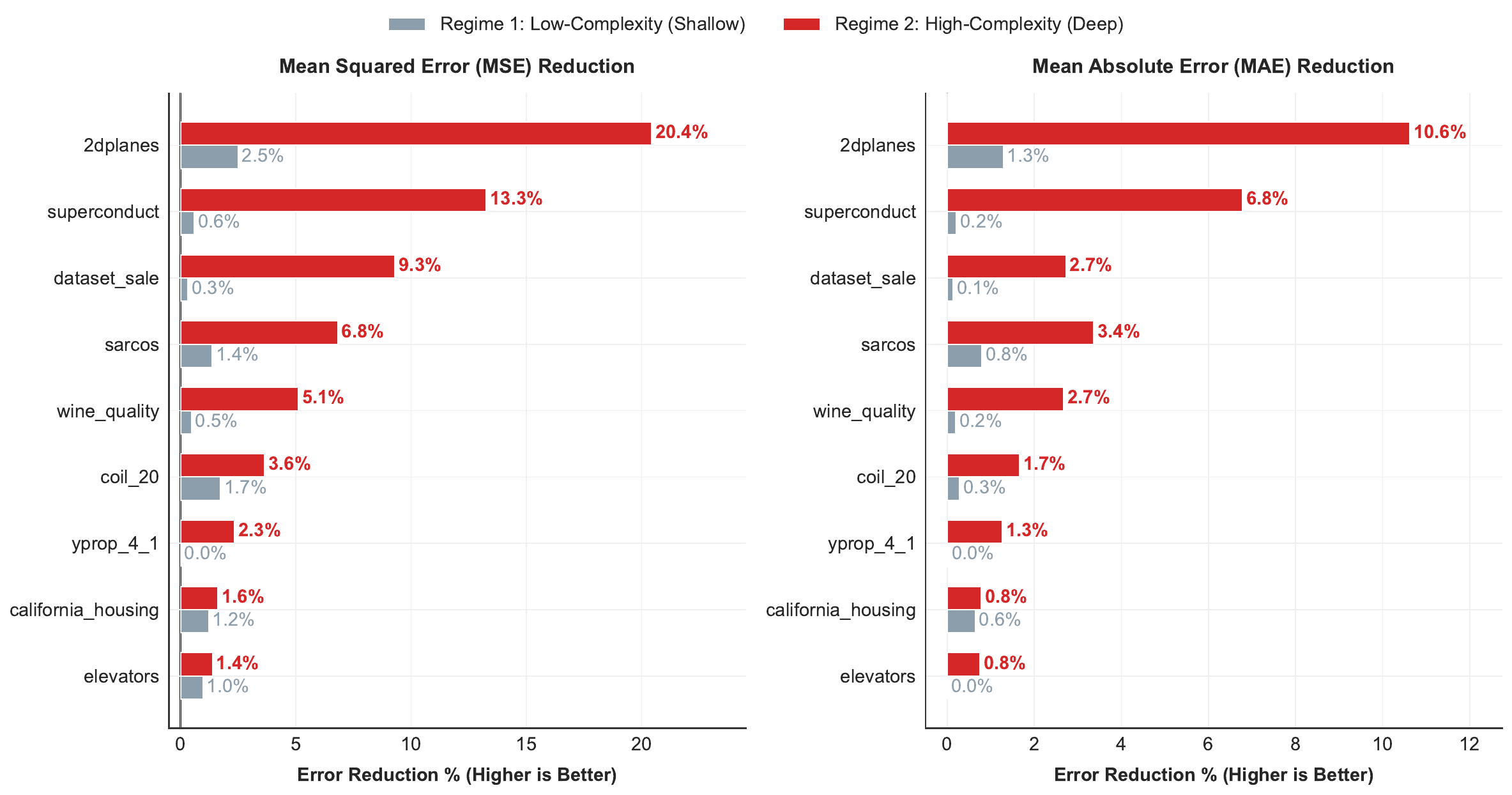}
    \caption{Relative error reduction in out-of-sample MSE (left) and MAE (right) achieved by the theoretically calibrated $\text{RF}^*$ (bootstrap with \texttt{max\_samples=}$\alpha^*$) over the standard $\text{RF}$ baseline (default bootstrap). The grey bars indicate marginal yet strictly non-negative improvements under the low-complexity regime. The red bars highlight substantial and consistent performance gains under the high-complexity regime, empirically validating that restricting the sampling fraction dynamically suppresses high-order interpolation noise even within a standard bootstrap framework.}
    \label{fig:error_reduction}
\end{figure}

Although the exact variance decomposition in \textcolor{red}{Theorem \ref{thm:exact_variance}} strictly assumes subagging to preserve multi-linear orthogonality, applying the CGAS-calibrated $\alpha^*$ to bootstrap sampling is theoretically justified. In the strongly regularized regime ($k \ll n$), bagging and subagging are asymptotically contiguous since the probability of duplicate draws vanishes \citep{buhlmann2002analyzing,wager2014asymptotic}. Even outside this asymptotic limit, restricting \texttt{max\_samples} bounds the expected number of unique instances per tree. This functionally induces a similar geometric attenuation on high-order interactions ($\zeta_c$), allowing $\alpha^*$ to remain an effective proxy for suppressing interpolation noise despite sample replacement.

\textcolor{red}{Figure \ref{fig:error_reduction}} compares the relative MSE and MAE reductions of $\text{RF}^*$ against the standard $\text{RF}$. In Regime 1, $\text{RF}^*$ yielded marginal performance changes (MSE reductions of $0\%$-$2.5\%$), demonstrating algorithmic safety by avoiding systematic degradation when rapid attenuation is unwarranted. In contrast, in Regime 2, $\text{RF}^*$ achieved substantial generalization improvements, reducing out-of-sample MSE by $20.4\%$ in \textit{2dplanes} and $13.3\%$ in \textit{superconduct}. These results align with \textcolor{red}{Theorem \ref{thm:spectral_filtering}}. Standard $\text{RF}$ bootstrap sampling draws approximately $63.2\%$ unique instances per tree, imposing a fixed expected subsampling fraction of $\alpha \approx 0.632$. While this attenuation profile suffices for the low-complexity Regime 1, it provides insufficient geometric penalty in Regime 2. Furthermore, the inherent sample duplication ($\sim 36.8\%$) may exacerbate idiosyncratic noise memorization. By restricting the sample size via \texttt{max\_samples=}$\alpha^*$, $\text{RF}^*$ enforces a stronger geometric attenuation, effectively controlling high-order variance and improving finite-sample generalization.

\section{Conclusion}
\label{sec:conclusion}

This work establishes the first exact, finite-sample variance identity for subagging by characterizing the resampling operator as a deterministic low-pass spectral filter. We prove that subagging geometrically attenuates high-order interaction variance, rendering three-decade-old static heuristics such as $\alpha \approx 0.632$ sub-optimal for high-capacity interpolators. These theoretical findings are implemented via the CGAS algorithm, which adaptively calibrates the subsampling ratio to the learner's complexity spectrum. To maintain multilinear orthogonality, the exact decomposition is restricted to sampling without replacement (subagging). Extending this framework to standard bootstrap aggregation remains a non-trivial challenge, as duplicate samples induce complex structural entanglements within the Hoeffding projections. Future research will focus on developing generalized orthogonal bases to disentangle these correlation structures, aiming for a complete finite-sample characterization of sampling with replacement.

\section*{Acknowledgements}
This work was supported in part by the National Natural Science Foundation of China under Grant 62403043. The authors declare that they have no known competing financial interests or personal relationships that could have appeared to influence the work reported in this paper.

\bibliography{sample}

\newpage

\appendix

\section{Notation}
\label{app:sec:notation}

To facilitate a rigorous finite-sample analysis and ensure mathematical clarity, we formalize the primary notation used throughout this paper in \textcolor{red}{Table \ref{tab:global_notation}}.

\begin{longtable}{p{0.15\textwidth} p{0.65\textwidth} p{0.15\textwidth}}
\caption{Summary of Global Symbols and Notations}
\label{tab:global_notation} \\
\toprule
\textbf{Symbol} & \textbf{Definition} & \textbf{Domain/Space} \\
\midrule
\endfirsthead

\multicolumn{3}{c}%
{{\bfseries \tablename\ \thetable{} -- continued from previous page}} \\
\toprule
\textbf{Symbol} & \textbf{Definition} & \textbf{Domain/Space} \\
\midrule
\endhead

\midrule \multicolumn{3}{r}{{Continued on next page}} \\ \midrule
\endfoot

\bottomrule
\endlastfoot

\multicolumn{3}{l}{\textit{\textbf{1. Sets and Spaces}}} \\
\midrule
$\mathcal{X}, \mathcal{Y}$ & Feature input space and target continuous output space & - \\
$\mathcal{Z}$ & The measurable product space $\mathcal{X} \times \mathcal{Y}$ & - \\
$\mathcal{D}_n$ & The finite training dataset $\{Z_1, \dots, Z_n\}$ & $\mathcal{Z}^n$ \\
$[n], [k]$ & The integer index sets $\{1, 2, \dots, n\}$ and $\{1, 2, \dots, k\}$ & $\mathbb{Z}^+$ \\
$\binom{[n]}{k}, \binom{[n]}{c}$ & The collection of all $k$-element or $c$-element subsets drawn from $[n]$ & $2^{[n]}$ \\
$\Omega$ & The discrete sample space of all possible subsamples of size $k$ & $2^{[n]}$ \\
$\mathcal{S}_k$ & The symmetric group containing all permutations of $k$ elements & - \\
$\mathcal{P}(C)$ & The Boolean lattice (power set) of interaction index subset $C$ & $2^C$ \\
$\mathcal{A}_0, \mathcal{A}_1$ & Sub-lattice partitions of $\mathcal{P}(C)$ based strictly on the exclusion ($\mathcal{A}_0$) or inclusion ($\mathcal{A}_1$) of a target integration index $i$ & $\subset 2^C$ \\
$\mathcal{H}$ & The functional space of all valid Hoeffding variance spectra & $\mathbb{R}^k_{\ge 0}$ \\

\midrule
\multicolumn{3}{l}{\textit{\textbf{2. Data and Random Variables}}} \\
\midrule
$Z_i, z_i$ & Independent and identically distributed (i.i.d.) random variable and its realization & $\mathcal{Z}$ \\
$S, S_1, S_2$ & Index subsets representing distinct subsamples of cardinality $k$ & $\binom{[n]}{k}$ \\
$C, J$ & Arbitrary interaction index subsets, typically of cardinality $c$ and $j$ & $\binom{[n]}{c}, 2^C$ \\
$C_1, C_2$ & Distinct interaction index subsets used to prove the mutual orthogonality of Hoeffding projections ($C_1 \neq C_2$) & $2^{[k]}$ \\
$A, B$ & Arbitrary non-empty index subsets of the training sample used in the exact algebraic variance expansion of the base learner & $2^{[k]} \setminus \{\emptyset\}$ \\
$Z_S, z_J$ & The specific subset of training random variables indexed by $S$ or $J$ & $\mathcal{Z}^{|S|}, \mathcal{Z}^{|J|}$ \\
$\xi$ & Internal algorithmic randomness (e.g., random feature splits) & - \\
$\pi$ & Random permutation of training data indices drawn uniformly & $\mathcal{S}_k$ \\
$E_j$ & The event that the $j$-th element of interaction target $C$ is included in subsample $S$ & - \\

\midrule
\multicolumn{3}{l}{\textit{\textbf{3. Scalars and Indices}}} \\
\midrule
$n, k$ & Total sample size and the subsample size drawn without replacement & $\mathbb{Z}^+$ \\
$c, m$ & Strict order of multi-sample interaction ($1 \le c \le k$) & $\mathbb{Z}^+$ \\
$\rho$ & The cardinality of intersection between two subsamples ($\rho = |S_1 \cap S_2|$) & $\{0, \dots, k\}$ \\
$M$ & The maximum finite interaction depth (cutoff order) / Structural capacity proxy & $\mathbb{Z}^+$ \\
$\alpha$ & The continuous subsampling ratio, defined as $\alpha = k/n$ & $(0, 1]$ \\
$\alpha^*$ & The optimal subsampling regularization parameter & $(0, 1]$ \\
$B_0, \beta$ & Strict upper-bound constants for the non-parametric bias decay rate & $\mathbb{R}^+$ \\
$d$ & The dimensionality of the input feature space $\mathcal{X}$ & $\mathbb{Z}^+$ \\
$e$ & The smoothness exponent of a Hölder-continuous target function & $\mathbb{R}^+$ \\
$p$ & Statistical p-value utilized in empirical hypothesis testing (e.g., Wilcoxon signed-rank test) & $(0, 1)$ \\
$Q$ & Algebraic bias decay exponent specific to random trees, dependent on the number of active features & $\mathbb{R}^+$ \\
$\sigma_1^2, \sigma_M^2$ & Invariant total variance contributions of the signal and high-order overfitting noise in the bi-modal case study & $\mathbb{R}^+ \cup \{0\}$ \\
$d_{\max}$ & Maximum tree depth used as an algorithmic proxy for structural complexity & $\mathbb{Z}^+ \cup \{\infty\}$ \\
$K$ & Number of internal cross-validation folds in CGAS algorithm & $\mathbb{Z}^+$ \\
$B_{\text{search}}, B_{\text{final}}$ & The number of estimators in the pilot and final subagged ensembles & $\mathbb{Z}^+$ \\

\midrule
\multicolumn{3}{l}{\textit{\textbf{4. Functions, Operators, and Estimators}}} \\
\midrule
$h(x; Z_S)$ & Deterministic, permutation-invariant base learning algorithm trained on $Z_S$ & $\mathbb{R}$ \\
$\mathcal{A}(x; Z_{1:k}, \xi, \pi)$ & Randomized or order-dependent base learning algorithm prior to marginalization & $\mathbb{R}$ \\
$\hat{F}_{n,k}(x)$ & Subagged ensemble estimator (incomplete, infinite-order U-statistic) & $\mathbb{R}$ \\
$f^*(x)$ & The true target function representing the conditional expectation $\mathbb{E}[Y|X=x]$ & $\mathbb{R}$ \\
$\theta(x), \theta_k(x)$ & Marginalized expected prediction of a base learner trained on $k$ samples & $\mathbb{R}$ \\
$h_c(x; z_C)$ & The $c$-th order canonical Hoeffding projection kernel (pure interactions) & $\mathbb{R}$ \\
$E_J(z_J)$ & Marginalized conditional expectation integrating out remaining variables & $\mathbb{R}$ \\
$H_c(\alpha)$ & The variance inflation multiplier for the $c$-th order spectrum & $\mathbb{R}^+$ \\
$f(\alpha, \boldsymbol{\zeta})$ & The continuous objective function $f(\alpha, \boldsymbol{\zeta}) = -\mathcal{M}(\alpha; \boldsymbol{\zeta})$ used for submodularity in Topkis's Theorem & $\mathbb{R}$ \\
$\Gamma(z)$ & The Gamma function utilized for the analytic continuation of discrete factorials & $\mathbb{R}$ \\
$\mathcal{O}(\cdot)$ & Big-O notation, representing the asymptotic upper bound of a functional growth rate & - \\

\midrule
\multicolumn{3}{l}{\textit{\textbf{5. Measures, Variance Components, and Objectives}}} \\
\midrule
$P, P^{\otimes k}, P^n$ & The unknown data-generating probability measure and its product measures & - \\
$\mathcal{U}, \mathbb{P}_{\mathcal{U}}$ & The uniform probability measure over the discrete subsampling space $\Omega$ & - \\
$\mathbb{E}, \text{Var}, \text{Cov}$ & Mathematical expectation, variance, and covariance operators & - \\
$\mathbb{I}_{\{\cdot\}}$ & The indicator function isolating structural inclusion events & $\{0, 1\}$ \\
$\zeta_c(x)$ & The $c$-th order Hoeffding variance component (Hoeffding spectrum / Sobol index) & $\mathbb{R}^+ \cup \{0\}$ \\
$\boldsymbol{\zeta}$ & The partially ordered variance complexity spectrum sequence $\{\zeta_c\}_{c=1}^M$ & $\mathbb{R}^M$ \\
$\gamma_c(n, k)$ & The exact combinatorial attenuation factor (spectral filter transfer function) & $(0, 1]$ \\
$\mathcal{B}(k)$ & The inherent finite-sample bias, strictly equal to $\theta_k(x) - f^*(x)$ & $\mathbb{R}$ \\
$\mathcal{M}(\alpha; \boldsymbol{\zeta}, M)$ & The continuous finite-sample Mean Squared Error (MSE) upper-bounding envelope & $\mathbb{R}^+$ \\
$\succeq$ & The partial order defined on the functional space of spectra $\mathcal{H}$ & - \\
\end{longtable}

\section{Missing Proof of Lemma \ref{lem:degeneracy_orthogonality}}
\label{app:proof_lem_degeneracy}

\begin{proof}
Let $C \subseteq \{1, \dots, k\}=[k]$ be an index subset with $|C| = c \ge 1$. For any $J \subseteq C$, let $E_{J}(z_{J}) \triangleq \mathbb{E} [ h(x; Z_{1:k}) \mid Z_{J} = z_{J} ]$ denote the marginalized conditional expectation. By the recursive construction in Eq.~\eqref{eq:canonical_projections}, the kernel $h_c$ satisfies the relation $E_C(z_C) = \sum_{J \subseteq C} h_{|J|}(x; z_{J})$. Applying M\"obius inversion over the Boolean lattice of subsets $(\mathcal{P}(C), \subseteq)$ \citep{stanley2011enumerative}, we obtain the explicit form:
\begin{equation}
\label{eq:mobius_explicit}
    h_c(x; z_C) = \sum_{J \subseteq C} (-1)^{|C| - |J|} E_{J}(z_{J}).
\end{equation}
To prove strong degeneracy, fix an arbitrary index $i \in C$ and consider the expectation of $h_c$ over $Z_i \sim P$. Applying $\mathbb{E}_{Z_i}$ to Eq.~\eqref{eq:mobius_explicit} and invoking the linearity of expectation:
\begin{equation}
\label{eq:exp_split}
    \mathbb{E}_{Z_i} [h_c(x; z_C)] = \sum_{J \subseteq C} (-1)^{|C| - |J|} \mathbb{E}_{Z_i} \left[ E_{J}(z_{J}) \right].
\end{equation}
Define the sub-lattice partition $\mathcal{P}(C) = \mathcal{A}_0 \cup \mathcal{A}_1$, where $\mathcal{A}_0 = \{ J \in \mathcal{P}(C) : i \notin J \}$ and $\mathcal{A}_1 = \{ J \in \mathcal{P}(C) : i \in J \}$. There exists a bijection $\phi: \mathcal{A}_0 \to \mathcal{A}_1$ defined by $\phi(J) = J \cup \{i\}$. We evaluate $\mathbb{E}_{Z_i} [E_J]$ for each case:
\begin{enumerate}
    \item For $J \in \mathcal{A}_0$: Since $Z_i \notin Z_J$, $E_J(z_J)$ is constant w.r.t. $Z_i$, thus $\mathbb{E}_{Z_i} [E_J(z_J)] = E_J(z_J)$.
    \item For $J \in \mathcal{A}_1$: Let $J = J' \cup \{i\}$ where $J' \in \mathcal{A}_0$. By the tower property of conditional expectation:
    \begin{equation*}
        \mathbb{E}_{Z_i} [E_{J' \cup \{i\}}(z_{J' \cup \{i\}})] = \mathbb{E}_{Z_i} [ \mathbb{E} [ h \mid Z_{J'}, Z_i ] ] = \mathbb{E} [ h \mid Z_{J'} ] = E_{J'}(z_{J'}).
    \end{equation*}
\end{enumerate}
Substituting these results and the bijection $\phi$ into Eq.~\eqref{eq:exp_split} yields:
\begin{align*}
    \mathbb{E}_{Z_i} [h_c(x; z_C)] 
    &= \sum_{J \in \mathcal{A}_0} (-1)^{|C| - |J|} \mathbb{E}_{Z_i} [E_J] + \sum_{J \in \mathcal{A}_1} (-1)^{|C| - |J|} \mathbb{E}_{Z_i} [E_J] \\
    &= \sum_{J \in \mathcal{A}_0} (-1)^{|C| - |J|} E_J(z_J) + \sum_{J \in \mathcal{A}_0} (-1)^{|C| - |J \cup \{i\}|} E_J(z_J) \\
    &= \sum_{J \in \mathcal{A}_0} \left( (-1)^{|C| - |J|} + (-1)^{|C| - |J| - 1} \right) E_J(z_J) \\
    &= \sum_{J \in \mathcal{A}_0} \left( (-1)^{|C| - |J|} - (-1)^{|C| - |J|} \right) E_J(z_J) = 0.
\end{align*}
This confirms that $h_c$ is strongly degenerate $P$-almost surely for all $i \in C$.

Consider two distinct index subsets $C_1, C_2 \subseteq \{1, \dots, k\}$ such that $C_1 \neq C_2$. Without loss of generality, assume there exists $i \in C_1 \setminus C_2$. Let $U = (C_1 \cup C_2) \setminus \{i\}$. The covariance evaluates to:
\begin{align*}
    \mathbb{E}_{Z_{C_1 \cup C_2}} \left[ h_{|C_1|}(x; Z_{C_1}) h_{|C_2|}(x; Z_{C_2}) \right] 
    &= \mathbb{E}_{Z_{U}} \left[ \mathbb{E}_{Z_i} \left[ h_{|C_1|}(x; Z_{C_1}) h_{|C_2|}(x; Z_{C_2}) \mid Z_U \right] \right] \\
    &\stackrel{(a)}{=} \mathbb{E}_{Z_{U}} \left[ h_{|C_2|}(x; Z_{C_2}) \cdot \mathbb{E}_{Z_i} \left[ h_{|C_1|}(x; Z_{C_1}) \mid Z_U \right] \right] \\
    &\stackrel{(b)}{=} \mathbb{E}_{Z_{U}} \left[ h_{|C_2|}(x; Z_{C_2}) \cdot 0 \right] = 0,
\end{align*}
where (a) follows since $i \notin C_2$, making $h_{|C_2|}$ independent of $Z_i$, and (b) follows from the strong degeneracy of $h_{|C_1|}$ w.r.t. $Z_i$. This completes the proof.
\end{proof}

\section{Missing Proof of Lemma \ref{lem:base_variance_decomp}}
\label{app:proof_lem_base_var}

\begin{proof}
By reversing the recursive definition of the canonical projections in Eq.~\eqref{eq:canonical_projections}, the centralized base learner can be expressed as an exact sum over all non-empty subsets of the index set $[k]$:
\begin{equation*}
    h(x; Z_{1:k}) - \theta(x) = \sum_{\substack{A \subseteq [k] \\ A \neq \emptyset}} h_{|A|}(x; Z_A).
\end{equation*}

By definition, the variance of the base learner is the expectation of its squared centered prediction. Expanding the square of the sum yields:
\begin{align*}
    \text{Var}\big(h(x; Z_{1:k})\big) &= \mathbb{E}_{Z_{1:k} \sim P^{\otimes k}} \left[ \left( \sum_{\substack{A \subseteq [k] \\ A \neq \emptyset}} h_{|A|}(x; Z_A) \right)^2 \right] \\[1ex]
    &= \mathbb{E}_{Z_{1:k} \sim P^{\otimes k}} \left[ \sum_{\substack{A \subseteq [k] \\ A \neq \emptyset}} \sum_{\substack{B \subseteq [k] \\ B \neq \emptyset}} h_{|A|}(x; Z_A) h_{|B|}(x; Z_B) \right].
\end{align*}

By the linearity of expectation over finite sums, we interchange the expectation and the double summation:
\begin{align*}
    \text{Var}\big(h(x; Z_{1:k})\big) &= \sum_{\substack{A \subseteq [k] \\ A \neq \emptyset}} \sum_{\substack{B \subseteq [k] \\ B \neq \emptyset}} \mathbb{E}_{Z_{1:k} \sim P^{\otimes k}} \Big[ h_{|A|}(x; Z_A) h_{|B|}(x; Z_B) \Big].
\end{align*}

According to \textcolor{red}{Lemma~\ref{lem:degeneracy_orthogonality}}, the canonical projections are mutually orthogonal. For any two distinct subsets $A \neq B$, their cross-covariance evaluates exactly to zero. Consequently, the double summation collapses strictly to its diagonal terms (where $A = B$):
\begin{align*}
    \text{Var}\big(h(x; Z_{1:k})\big) &= \sum_{\substack{A \subseteq [k] \\ A \neq \emptyset}} \mathbb{E}_{Z_{1:k} \sim P^{\otimes k}} \Big[ h_{|A|}(x; Z_A)^2 \Big] \\[1ex]
    &= \sum_{\substack{A \subseteq [k] \\ A \neq \emptyset}} \text{Var}\big(h_{|A|}(x; Z_A)\big).
\end{align*}

To evaluate this single summation analytically, we partition the terms based on the subset cardinality $c = |A|$:
\begin{align*}
    \text{Var}\big(h(x; Z_{1:k})\big) &= \sum_{c=1}^k \sum_{\substack{A \subseteq [k] \\ |A| = c}} \text{Var}\big(h_c(x; Z_A)\big).
\end{align*}

Under \textcolor{red}{Assumption~\ref{ass:sym_sq_int}} and the i.i.d. measure $P^{\otimes k}$, the variance of a canonical projection depends strictly on its interaction order $c$, independent of the specific indices in subset $A$. Thus, we substitute $\text{Var}(h_c(x; Z_A)) = \zeta_c(x)$:
\begin{align*}
    \text{Var}\big(h(x; Z_{1:k})\big) &= \sum_{c=1}^k \sum_{\substack{A \subseteq [k] \\ |A| = c}} \zeta_c(x).
\end{align*}

Since the inner summation iterates over a constant $\zeta_c(x)$ for all possible $c$-element subsets drawn from $[k]$, it resolves to a simple count. There are exactly $\binom{k}{c}$ such distinct subsets, yielding:
\begin{align*}
    \text{Var}\big(h(x; Z_{1:k})\big) &= \sum_{c=1}^k \binom{k}{c} \zeta_c(x).
\end{align*}
This completes the proof.
\end{proof}

\section{Missing Proof for Theorem \ref{thm:exact_variance}}
\label{app:proof_thm_exa_var}

\begin{proof}
Let $[n] = \{1, 2, \dots, n\}$ denote the finite index set of the training dataset. We define the discrete subsampling space $\Omega = \{S \subseteq [n] : |S| = k\}$. The subagged estimator $\hat{F}_{n,k}(x)$ induces a uniform probability measure $\mathcal{U}$ over $\Omega$, where $\mathbb{P}_{\mathcal{U}}(S) = \binom{n}{k}^{-1}$ for any realization $S \in \Omega$.

The total variance of $\hat{F}_{n,k}(x)$ is given by the expected covariance of two base learners trained on index sets $S_1, S_2$ drawn independently from $\Omega$ under the measure $\mathcal{U}$. Note that the covariance is evaluated with respect to the unknown data-generating distribution, while the outer expectation is with respect to the algorithmic resampling:
\begin{equation}
\label{eq:var_expansion}
    \text{Var}(\hat{F}_{n,k}(x)) = \mathbb{E}_{S_1, S_2 \sim \mathcal{U}} \left[ \text{Cov}(h(x; Z_{S_1}), h(x; Z_{S_2})) \right].
\end{equation}

While \textcolor{red}{Lemma~\ref{lem:degeneracy_orthogonality}} establishes the orthogonality of interaction subsets $C \subseteq [k]$ within a single base learner, we now lift this analysis to the entire finite population $[n]$. Consider a specific interaction target subset $C \subseteq [n]$ with $|C|=c$. By the established Hoeffding multi-linear orthogonality, the functional covariance between base learners trained on different subsamples $S_1$ and $S_2$ is completely decoupled into independent orthogonal contributions. A specific $c$-th order component $\zeta_c(x)$ contributes to the covariance if and only if its underlying index set $C$ is fully contained in the intersection of both subsamples, i.e., $C \subseteq (S_1 \cap S_2)$. Using the indicator function $\mathbb{I}_{\{C \subseteq S\}}$, we write:
\begin{equation*}
    \text{Cov}(h(x; Z_{S_1}), h(x; Z_{S_2})) = \sum_{c=1}^k \sum_{\substack{C \subseteq [n] \\ |C|=c}} \zeta_c(x) \mathbb{I}_{\{C \subseteq S_1\}} \mathbb{I}_{\{C \subseteq S_2\}}.
\end{equation*}

Substituting this into Eq.~\eqref{eq:var_expansion}. By invoking the linearity of expectation for finite sums and exploiting the unconditional independence of the subsampling events $S_1$ and $S_2$, we obtain:
\begin{align*}
    \text{Var}(\hat{F}_{n,k}(x)) &= \mathbb{E}_{S_1, S_2 \sim \mathcal{U}} \left[ \sum_{c=1}^k \sum_{|C|=c} \zeta_c(x) \mathbb{I}_{\{C \subseteq S_1\}} \mathbb{I}_{\{C \subseteq S_2\}} \right] \\[1ex]
    &= \sum_{c=1}^k \zeta_c(x) \sum_{|C|=c} \mathbb{E}_{S_1 \sim \mathcal{U}}[\mathbb{I}_{\{C \subseteq S_1\}}] \cdot \mathbb{E}_{S_2 \sim \mathcal{U}}[\mathbb{I}_{\{C \subseteq S_2\}}] \\[1ex]
    &= \sum_{c=1}^k \zeta_c(x) \sum_{|C|=c} \left( \mathbb{P}_{\mathcal{U}}(C \subseteq S) \right)^2.
\end{align*}

For a fixed target subset $C$ of size $c$, constructing a realization $S \in \Omega$ such that $C \subseteq S$ necessitates fixing the $c$ elements of $C$ and uniformly choosing the remaining $k-c$ elements from the complement set. Thus:
\begin{equation*}
    \mathbb{P}_{\mathcal{U}}(C \subseteq S) = \frac{|\{ S \in \Omega \mid C \subseteq S \}|}{|\Omega|} = \frac{\binom{n-c}{k-c}}{\binom{n}{k}}.
\end{equation*}

Since there are exactly $\binom{n}{c}$ possible target subsets of size $c$ within $[n]$, the inner summation evaluates explicitly as:
\begin{align*}
    \sum_{|C|=c} \left( \mathbb{P}_{\mathcal{U}}(C \subseteq S) \right)^2 
    &= \sum_{|C|=c} \left( \frac{\binom{n-c}{k-c}}{\binom{n}{k}} \right)^2 \\[1ex]
    &= \binom{n}{c} \left( \frac{\binom{n-c}{k-c}}{\binom{n}{k}} \right)^2 \\[1ex]
    &= \left[ \frac{\binom{n}{c} \binom{n-c}{k-c}}{\binom{n}{k}} \right] \cdot \frac{\binom{n-c}{k-c}}{\binom{n}{k}} \\[2ex]
    &= \left[ \frac{ \frac{n!}{c!(n-c)!} \cdot \frac{(n-c)!}{(k-c)!(n-k)!} }{ \frac{n!}{k!(n-k)!} } \right] \cdot \left( \frac{ \frac{(n-c)!}{(k-c)!(n-k)!} }{ \frac{n!}{k!(n-k)!} } \right) \\[2ex]
    &= \left[ \frac{n!}{c!(k-c)!(n-k)!} \cdot \frac{k!(n-k)!}{n!} \right] \cdot \left( \frac{(n-c)! \, k!}{(k-c)! \, n!} \right) \\[2ex]
    &= \left[ \frac{k!}{c!(k-c)!} \right] \cdot \left( \frac{k(k-1)\cdots(k-c+1)}{n(n-1)\cdots(n-c+1)} \right) \\[2ex]
    &= \binom{k}{c} \prod_{j=0}^{c-1} \frac{k-j}{n-j}.
\end{align*}

By substituting this identity back into the variance expansion and defining the attenuation factor as $\gamma_c(n,k) \triangleq \prod_{j=0}^{c-1} \frac{k-j}{n-j}$, we recover the exact decoupling:
\begin{equation*}
    \text{Var}(\hat{F}_{n,k}(x)) = \sum_{c=1}^k \gamma_c(n,k) \binom{k}{c} \zeta_c(x).
\end{equation*}
This completes the proof.
\end{proof}

\section{Missing Proof of Theorem \ref{thm:spectral_filtering}}
\label{app:proof_thm_spe_fil}

\begin{proof}
Let $S$ be a random subsample of size $k$ drawn uniformly without replacement from $[n]$. For a fixed subset $C = \{i_1, \dots, i_c\}$ of size $c$, let $E_j$ denote the event $\{i_j \in S\}$ for $j=1, \dots, c$. By the symmetry of the uniform measure $\mathcal{U}$, the attenuation factor is the joint probability:
\begin{equation*}
    \gamma_c(n, k) = \mathbb{P}(E_1 \cap E_2 \cap \dots \cap E_c).
\end{equation*}
Applying the chain rule of probability:
\begin{equation}
\label{eq:proof3_chain_rule}
    \gamma_c(n, k) = \mathbb{P}(E_1) \prod_{j=1}^{c-1} \mathbb{P}(E_{j+1} \mid E_1, \dots, E_j).
\end{equation}
The marginal probability is $\mathbb{P}(E_1) = \frac{k}{n} = \alpha$. For $j \ge 1$, the conditional probability is:
\begin{equation*}
    \mathbb{P}(E_{j+1} \mid E_1, \dots, E_j) = \frac{k - j}{n - j}.
\end{equation*}
To compare this with the marginal ratio $\alpha$, consider the difference:
\begin{align*}
    \frac{k - j}{n - j} - \frac{k}{n} &= \frac{n(k - j) - k(n - j)}{n(n - j)} \nonumber \\
    &= \frac{nk - nj - kn + kj}{n(n - j)} \nonumber \\
    &= \frac{j(k - n)}{n(n - j)}. 
\end{align*}
Under the condition $n > k$, the numerator $j(k-n)$ is strictly negative for all $j \ge 1$. Thus:
\begin{equation}
\label{eq:proof3_neg_association}
    \mathbb{P}(E_{j+1} \mid E_1, \dots, E_j) < \mathbb{P}(E_1) = \alpha.
\end{equation}
Substituting Eq.~\eqref{eq:proof3_neg_association} into the chain rule Eq.~\eqref{eq:proof3_chain_rule} for $c \ge 2$:
\begin{equation*}
    \gamma_c(n, k) < \alpha \prod_{j=1}^{c-1} \alpha = \alpha^c.
\end{equation*}
The equality $\gamma_c(n, k) = \alpha^c$ holds if $c=1$ (single event) or if $n=k$ (where $\mathbb{P}(E_{j+1} \mid \dots) = 1$). For all other cases, $\gamma_c(n,k) < \alpha^c$.

For a fixed interaction order $c$, we analyze the limit of the product form:
\begin{align}
    \lim_{n \to \infty, \frac{k}{n} \to \alpha} \gamma_c(n, k) &= \lim_{n \to \infty} \prod_{j=0}^{c-1} \frac{k - j}{n - j} \nonumber \\
    &= \prod_{j=0}^{c-1} \left( \lim_{n \to \infty} \frac{k - j}{n - j} \right). \label{eq:proof3_limit_prod}
\end{align}
Expanding each term in the limit:
\begin{align}
\label{eq:proof3_limit_term}
    \lim_{n \to \infty, \frac{k}{n} \to \alpha} \frac{k - j}{n - j} &= \lim_{n \to \infty, \frac{k}{n} \to \alpha} \frac{k/n - j/n}{1 - j/n} \nonumber \\
    &= \frac{\lim (k/n) - \lim (j/n)}{1 - \lim (j/n)} \nonumber \\
    &= \frac{\alpha - 0}{1 - 0} = \alpha.
\end{align}
Substituting Eq.~\eqref{eq:proof3_limit_term} back into the product Eq.~\eqref{eq:proof3_limit_prod}:
\begin{equation*}
    \lim_{n \to \infty, \frac{k}{n} \to \alpha} \gamma_c(n, k) = \prod_{j=0}^{c-1} \alpha = \alpha^c.
\end{equation*}
This demonstrates that as $n \to \infty$, the negative association between sampling events vanishes, and the attenuation converges to a pure geometric filter.
\end{proof}

\section{Missing Proof of Lemma \ref{lem:mse_bound}}
\label{app:proof_lem_mse_bound}

\begin{proof}
Let $\mathcal{D}_n \sim P^n$ denote the training dataset and $\mathcal{U}$ denote the uniform probability measure over the discrete subsampling space $\Omega = \{S \subseteq [n] : |S| = k\}$. The subagged estimator is defined as the expectation of the base learner over the measure $\mathcal{U}$:
\begin{equation*}
\hat{F}_{n,k}(x) = \mathbb{E}_{S \sim \mathcal{U}}[h(x; Z_S)].
\end{equation*}

The point-wise MSE at a target point $x$ is defined as the $L_2(P^n)$ distance between the subagged estimator and the true target function $f^*(x)$:
\begin{equation*}
\text{MSE}(\hat{F}_{n,k}(x)) = \mathbb{E}_{\mathcal{D}_n \sim P^n} \left[ \left(\hat{F}_{n,k}(x) - f^*(x)\right)^2 \right].
\end{equation*}

To isolate the structural bias from the sampling variability, we perform a standard bias-variance decomposition by centering the estimator around its expectation $\mathbb{E}_{\mathcal{D}_n}[\hat{F}_{n,k}(x)]$:
\begin{equation*}
\begin{aligned}
\text{MSE}(\hat{F}_{n,k}(x)) &= \mathbb{E}_{\mathcal{D}_n} \left[ \left( \hat{F}_{n,k}(x) - \mathbb{E}_{\mathcal{D}_n}[\hat{F}_{n,k}(x)] + \mathbb{E}_{\mathcal{D}_n}[\hat{F}_{n,k}(x)] - f^*(x) \right)^2 \right] \\
&= \mathbb{E}_{\mathcal{D}_n} \left[ \left( \hat{F}_{n,k}(x) - \mathbb{E}_{\mathcal{D}_n}[\hat{F}_{n,k}(x)] \right)^2 \right] + \left( \mathbb{E}_{\mathcal{D}_n}[\hat{F}_{n,k}(x)] - f^*(x) \right)^2 \\
&\quad + 2 \, \mathbb{E}_{\mathcal{D}_n} \left[ \left( \hat{F}_{n,k}(x) - \mathbb{E}_{\mathcal{D}_n}[\hat{F}_{n,k}(x)] \right) \left( \mathbb{E}_{\mathcal{D}_n}[\hat{F}_{n,k}(x)] - f^*(x) \right) \right].
\end{aligned}
\end{equation*}

Since the term $\big(\mathbb{E}_{\mathcal{D}_n}[\hat{F}_{n,k}(x)] - f^*(x)\big)$ is deterministic with respect to the data-generating measure $P^n$, by the linearity of expectation, the cross-term simplifies to:
\begin{equation*}
\begin{aligned}
\text{Cross-Term} &= 2 \left( \mathbb{E}_{\mathcal{D}_n}[\hat{F}_{n,k}(x)] - f^*(x) \right) \cdot \left( \mathbb{E}_{\mathcal{D}_n} \left[ \hat{F}_{n,k}(x) \right] - \mathbb{E}_{\mathcal{D}_n} \left[ \mathbb{E}_{\mathcal{D}_n}[\hat{F}_{n,k}(x)] \right] \right) \\
&= 2 \left( \mathbb{E}_{\mathcal{D}_n}[\hat{F}_{n,k}(x)] - f^*(x) \right) \cdot \left( \mathbb{E}_{\mathcal{D}_n}[\hat{F}_{n,k}(x)] - \mathbb{E}_{\mathcal{D}_n}[\hat{F}_{n,k}(x)] \right) \\
&= 0.
\end{aligned}
\end{equation*}

This leaves the exact additive decomposition into squared bias and ensemble variance:
\begin{equation*}
\text{MSE}(\hat{F}_{n,k}(x)) = \underbrace{ \left( \mathbb{E}_{\mathcal{D}_n}[\hat{F}_{n,k}(x)] - f^*(x) \right)^2 }_{\text{Squared Bias}} + \underbrace{ \text{Var}_{\mathcal{D}_n}(\hat{F}_{n,k}(x)) }_{\text{Ensemble Variance}}.
\end{equation*}

For the \textbf{Squared Bias} term, we evaluate the inner expectation $\mathbb{E}_{\mathcal{D}_n}[\hat{F}_{n,k}(x)]$. By the linearity of expectation and the independence of the algorithmic resampling from the data-generating process, we interchange the order of expectations:
\begin{equation*}
\mathbb{E}_{\mathcal{D}_n \sim P^n}[\hat{F}_{n,k}(x)] = \mathbb{E}_{\mathcal{D}_n \sim P^n} \left[ \mathbb{E}_{S \sim \mathcal{U}}[h(x; Z_S)] \right] = \mathbb{E}_{S \sim \mathcal{U}} \left[ \mathbb{E}_{\mathcal{D}_n \sim P^n}[h(x; Z_S)] \right].
\end{equation*}

Because all training samples in $\mathcal{D}_n$ are independent and identically distributed (i.i.d.), the marginal expectation of the base learner trained on any subset $Z_S$ of size $|S|=k$ is invariant and equivalent to the expectation on the first $k$ samples:
\begin{equation*}
\mathbb{E}_{\mathcal{D}_n \sim P^n}[h(x; Z_S)] = \mathbb{E}_{Z_{1:k} \sim P^k}[h(x; Z_{1:k})] \triangleq \theta_k(x).
\end{equation*}
Thus, $\mathbb{E}_{S \sim \mathcal{U}} [ \theta_k(x) ] = \theta_k(x)$. Substituting this result into the bias component, we find that the subagging operator preserves the structural bias of the base learner:
\begin{equation*}
\text{Squared Bias}(\hat{F}_{n,k}(x)) = \left( \theta_k(x) - f^*(x) \right)^2.
\end{equation*}
By \textcolor{red}{Assumption \ref{ass:mon_bias_decay}}, this term is strictly upper bounded as $\text{Squared Bias} \le B_0 k^{-2\beta}$.

For the \textbf{Ensemble Variance} term, we directly invoke the exact finite-sample variance decomposition from \textcolor{red}{Theorem \ref{thm:exact_variance}}:
\begin{equation*}
\text{Var}_{\mathcal{D}_n}(\hat{F}_{n,k}(x)) = \sum_{c=1}^k \gamma_c(n,k) \binom{k}{c} \zeta_c(x).
\end{equation*}

Combining the upper bound on bias with the exact variance identity yields the final finite-sample MSE bound:
\begin{equation*}
\text{MSE}(\hat{F}_{n,k}(x)) \le B_0 k^{-2\beta} + \sum_{c=1}^k \gamma_c(n,k) \binom{k}{c} \zeta_c(x).
\end{equation*}
This completes the proof.
\end{proof}

\section{Missing Proof of Lemma \ref{lem:continuous_mse}}
\label{app:proof_lem_continuous_mse}

\begin{proof}
Recall from \textcolor{red}{Lemma~\ref{lem:mse_bound}} that the finite-sample MSE for any discrete $k \in \mathbb{Z}^+$ satisfies:
\begin{equation*}
    MSE(\hat{F}_{n,k}(x)) \le B_0 k^{-2\beta} + \sum_{c=1}^{k} \gamma_c(n, k) \binom{k}{c} \zeta_c(x).
\end{equation*}

Under \textcolor{red}{Assumption~\ref{ass:finite_interaction}}, the spectrum is strictly truncated at the maximum interaction depth $M$, i.e., $\zeta_c(x) = 0$ for all $c > M$. Substituting this structural sparsity into the variance summation yields:
\begin{equation*} 
    \sum_{c=1}^{k} \gamma_c(n, k) \binom{k}{c} \zeta_c(x) = \sum_{c=1}^{\min(k, M)} \gamma_c(n, k) \binom{k}{c} \zeta_c(x).
\end{equation*}

To construct a continuous surrogate envelope, we define the subsampling ratio $\alpha = k/n$. Constraining the domain to $\alpha \in [M/n, 1]$ directly establishes the lower bound $\alpha n \ge M$, which simplifies the truncated limit: $\min(\alpha n, M) = M$. Substituting $k = \alpha n$ into the MSE bound:
\begin{equation*}
    MSE(\hat{F}_{n,\alpha n}(x)) \le B_0 (\alpha n)^{-2\beta} + \sum_{c=1}^{M} \gamma_c(n, \alpha n) \binom{\alpha n}{c} \zeta_c(x).
\end{equation*}

Applying the spectral filtering bound $\gamma_c(n, \alpha n) \le \alpha^c$ from \textcolor{red}{Theorem~\ref{thm:spectral_filtering}}:
\begin{equation} 
\label{eq:proof_pre_gamma}
    MSE(\hat{F}_{n,\alpha n}(x)) \le B_0 (\alpha n)^{-2\beta} + \sum_{c=1}^{M} \alpha^c \binom{\alpha n}{c} \zeta_c(x).
\end{equation}

To establish differentiability over the continuum $\mathbb{R}^+$, we extend the discrete binomial coefficient via the falling factorial expansion:
\begin{equation*}
    \binom{\alpha n}{c} = \frac{1}{c!} \prod_{j=0}^{c-1} (\alpha n - j) = \frac{(\alpha n)(\alpha n - 1) \cdots (\alpha n - c + 1)}{c!}.
\end{equation*}

Multiplying the numerator and denominator by the Gamma function $\Gamma(\alpha n - c + 1)$ yields:
\begin{equation*} 
    \binom{\alpha n}{c} = \frac{(\alpha n)(\alpha n - 1) \cdots (\alpha n - c + 1) \cdot \Gamma(\alpha n - c + 1)}{c! \cdot \Gamma(\alpha n - c + 1)}.
\end{equation*}

Applying the functional equation $\Gamma(z+1) = z\Gamma(z)$ recursively absorbs the falling factorial:
\begin{align*}
    (\alpha n - c + 1)\Gamma(\alpha n - c + 1) &= \Gamma(\alpha n - c + 2) \\
    (\alpha n - c + 2)\Gamma(\alpha n - c + 2) &= \Gamma(\alpha n - c + 3) \\
    &\;\;\vdots \\
    (\alpha n)\Gamma(\alpha n) &= \Gamma(\alpha n + 1).
\end{align*}

Substituting this absorbed numerator and noting $\Gamma(c+1) = c!$, we obtain the analytic continuation:
\begin{equation} 
\label{eq:proof_gamma_final}
    \binom{\alpha n}{c} = \frac{\Gamma(\alpha n + 1)}{\Gamma(c + 1)\Gamma(\alpha n - c + 1)}.
\end{equation}

To guarantee that Eq.~\eqref{eq:proof_gamma_final} is globally analytic and strictly avoids the singular poles of the Gamma function at non-positive integers ($\{0, -1, -2, \dots\}$), we evaluate the algebraic lower bound of the denominator's argument. Given $c \le M$ and $\alpha n \ge M$:
\begin{equation*}
    \alpha n - c + 1 \ge M - M + 1 = 1 > 0.
\end{equation*}
Because the argument is strictly positive, $\Gamma(\alpha n - c + 1) \in \mathbb{R}^+$ uniformly across the continuous domain, verifying the absence of singularities.

Substituting the continuous mapping Eq.~\eqref{eq:proof_gamma_final} into Eq.~\eqref{eq:proof_pre_gamma} defines the exact upper-bounding differentiable surrogate:
\begin{equation*}
    MSE(\hat{F}_{n,\alpha n}(x)) \le B_0(\alpha n)^{-2\beta} + \sum_{c=1}^{M} \alpha^c \frac{\Gamma(\alpha n + 1)}{\Gamma(c + 1)\Gamma(\alpha n - c + 1)} \zeta_c(x) \triangleq \mathcal{M}(\alpha; \boldsymbol{\zeta}, M).
\end{equation*}
This completes the proof.
\end{proof}

\section{Missing Proof of Theorem \ref{thm:optimal_alpha}}
\label{app:proof_thm_opt_alpha}

\begin{proof}
Let $\mathcal{H}$ be the space of valid Hoeffding variance spectra. For any two algorithms with spectra $\boldsymbol{\zeta}^{(1)}, \boldsymbol{\zeta}^{(2)} \in \mathcal{H}$, we define the partial order $\succeq$ such that $\boldsymbol{\zeta}^{(2)} \succeq \boldsymbol{\zeta}^{(1)}$ if there exists a critical interaction order $c^*$ satisfying:
\begin{align*}
    \zeta_c^{(2)} - \zeta_c^{(1)} &= 0, \quad \forall c < c^*, \\
    \zeta_c^{(2)} - \zeta_c^{(1)} &\ge 0, \quad \forall c \ge c^*,
\end{align*}
with strict inequality $\zeta_c^{(2)} - \zeta_c^{(1)} > 0$ for at least one $c \ge c^*$.

The continuous asymptotic MSE envelope is given by:
\begin{equation*}
    \mathcal{M}(\alpha; \boldsymbol{\zeta}) = B_0 n^{-2\beta} \alpha^{-2\beta} + \sum_{c=1}^{n} \alpha^c \binom{\alpha n}{c} \zeta_c.
\end{equation*}

To apply the standard maximization form of Topkis's Theorem \citep{topkis1978,topkis1998}, we define the objective function as $f(\alpha, \boldsymbol{\zeta}) = -\mathcal{M}(\alpha; \boldsymbol{\zeta})$. The optimal subsampling ratio satisfies:
\begin{equation*}
    \alpha^*(\boldsymbol{\zeta}) = \arg\max_{\alpha \in (0, 1]} f(\alpha, \boldsymbol{\zeta}).
\end{equation*}

Let $H_c(\alpha)$ denote the variance inflation multiplier for the $c$-th order spectrum:
\begin{equation*}
    H_c(\alpha) \triangleq \alpha^c \binom{\alpha n}{c} = \frac{\alpha^c}{c!} \prod_{j=0}^{c-1} (\alpha n - j).
\end{equation*}

We evaluate the decreasing differences of $f(\alpha, \boldsymbol{\zeta})$ with respect to $\alpha$ and $\boldsymbol{\zeta}$. Taking the partial derivative of $f$ with respect to $\alpha$:
\begin{equation*}
    \frac{\partial f}{\partial \alpha}(\alpha, \boldsymbol{\zeta}) = 2\beta B_0 n^{-2\beta} \alpha^{-2\beta-1} - \sum_{c=1}^{n} \zeta_c H'_c(\alpha).
\end{equation*}

Evaluating the difference in marginal returns as the spectrum shifts from $\boldsymbol{\zeta}^{(1)}$ to $\boldsymbol{\zeta}^{(2)}$ (where $\boldsymbol{\zeta}^{(2)} \succeq \boldsymbol{\zeta}^{(1)}$), the bias terms perfectly cancel out:
\begin{align*}
    \frac{\partial f}{\partial \alpha}(\alpha, \boldsymbol{\zeta}^{(2)}) - \frac{\partial f}{\partial \alpha}(\alpha, \boldsymbol{\zeta}^{(1)}) 
    &= \left( 2\beta B_0 n^{-2\beta} \alpha^{-2\beta-1} - \sum_{c=1}^{n} \zeta^{(2)}_c H'_c(\alpha) \right) \\
    &\quad - \left( 2\beta B_0 n^{-2\beta} \alpha^{-2\beta-1} - \sum_{c=1}^{n} \zeta^{(1)}_c H'_c(\alpha) \right) \\
    &= -\sum_{c=1}^n \left( \zeta^{(2)}_c - \zeta^{(1)}_c \right) H'_c(\alpha) \\
    &= -\sum_{c=c^*}^n \left( \zeta^{(2)}_c - \zeta^{(1)}_c \right) H'_c(\alpha).
\end{align*}

To determine the sign of $H'_c(\alpha)$, we compute the first derivative of $H_c(\alpha)$ via logarithmic differentiation. The natural logarithm of $H_c(\alpha)$ is:
\begin{equation*}
    \ln H_c(\alpha) = c \ln \alpha - \ln(c!) + \sum_{j=0}^{c-1} \ln(\alpha n - j).
\end{equation*}
Differentiating with respect to $\alpha$:
\begin{align*}
    \frac{d}{d \alpha} \ln H_c(\alpha) &= \frac{d}{d \alpha} (c \ln \alpha) - \frac{d}{d \alpha} \ln(c!) + \sum_{j=0}^{c-1} \frac{d}{d \alpha} \ln(\alpha n - j) \\
    &= \frac{c}{\alpha} + 0 + \sum_{j=0}^{c-1} \frac{n}{\alpha n - j}.
\end{align*}
Thus, the derivative $H'_c(\alpha)$ is expanded as:
\begin{align*}
    H'_c(\alpha) &= H_c(\alpha) \frac{d \ln H_c(\alpha)}{d \alpha} \\
    &= \left[ \frac{\alpha^c}{c!} \prod_{j=0}^{c-1} (\alpha n - j) \right] \left( \frac{c}{\alpha} + \sum_{j=0}^{c-1} \frac{n}{\alpha n - j} \right).
\end{align*}

In the operative subsampling regime, $k = \alpha n \ge c$. Therefore, for all valid $j \in \{0, 1, \dots, c-1\}$, we have:
\begin{equation*}
    \alpha n - j \ge \alpha n - (c - 1) > 0.
\end{equation*}
Since $\alpha \in (0, 1]$ and $c \ge 1$, we strictly guarantee:
\begin{equation*}
    \frac{\alpha^c}{c!} \prod_{j=0}^{c-1} (\alpha n - j) > 0, \quad \frac{c}{\alpha} > 0, \quad \text{and} \quad \frac{n}{\alpha n - j} > 0.
\end{equation*}
Consequently, for all $c \ge 1$:
\begin{equation*}
    H'_c(\alpha) > 0.
\end{equation*}

Given $\boldsymbol{\zeta}^{(2)} \succeq \boldsymbol{\zeta}^{(1)}$, we have $\zeta^{(2)}_c - \zeta^{(1)}_c \ge 0$ for all $c \ge c^*$, with strict inequality $\zeta^{(2)}_c - \zeta^{(1)}_c > 0$ for at least one $c \ge c^*$. It follows that:
\begin{equation*}
    \sum_{c=c^*}^n \left( \zeta^{(2)}_c - \zeta^{(1)}_c \right) H'_c(\alpha) > 0.
\end{equation*}
Therefore, the marginal return difference is strictly negative:
\begin{equation*}
    \frac{\partial f}{\partial \alpha}(\alpha, \boldsymbol{\zeta}^{(2)}) - \frac{\partial f}{\partial \alpha}(\alpha, \boldsymbol{\zeta}^{(1)}) < 0.
\end{equation*}
This establishes that $f(\alpha, \boldsymbol{\zeta})$ satisfies strictly decreasing differences in $(\alpha, \boldsymbol{\zeta})$, making it strictly submodular in $(\alpha, \boldsymbol{\zeta})$. 

By Topkis's Theorem \citep{topkis1978,topkis1998}, the submodularity implies that the set of optimal solutions $\alpha^*(\boldsymbol{\zeta})$ is monotonically non-increasing in $\boldsymbol{\zeta}$ with respect to the strong set order, yielding 
\begin{equation*}
    \alpha^*_{(2)} \le \alpha^*_{(1)}. 
\end{equation*}
This concludes the proof.
\end{proof}
\end{document}